\documentclass[Afour,sageh,times]{sagej}
\usepackage{moreverb,url}
\usepackage{xcolor} 

\usepackage[colorlinks,bookmarksopen,bookmarksnumbered,citecolor=red,urlcolor=red]{hyperref}

\newcommand\BibTeX{{\rmfamily B\kern-.05em \textsc{i\kern-.025em b}\kern-.08em
T\kern-.1667em\lower.7ex\hbox{E}\kern-.125emX}}

\def\volumeyear{2016}

\usepackage{mathrsfs}
 
\usepackage{algorithm}
\usepackage{algorithmicx}
\usepackage[noend]{algpseudocode}

\usepackage{paralist}
\usepackage{centernot}

\usepackage{changes}
\usepackage{amsmath,amssymb}
\usepackage{amsthm}
\usepackage{mathtools}
\usepackage{graphicx}
\usepackage{caption}
\usepackage{subcaption}
\usepackage{tikz}
\usetikzlibrary{automata,positioning}
\usepackage{acronym}
\usepackage{setspace}
\usepackage[colorlinks,bookmarksopen,bookmarksnumbered,citecolor=red,urlcolor=red]{hyperref}

\usepackage{makecell, booktabs}

\newif\ifuseboldmathops
\newif\ifuseittextabbrevs
\useboldmathopstrue   

\ifuseittextabbrevs

	\newcommand{\ie}{{\it i.e.}}

\else

	\newcommand{\ie}{i.e.}

\fi

\ifuseboldmathops

\else

\fi

\ifuseboldmathops

\else

\fi

\ifuseboldmathops

\else

\fi

\ifuseboldmathops

\else

\fi

\newcommand{\argmax}{\mathop{\mathrm{argmax}}}

\newcommand{\truev}{\mathsf{true}}
\newcommand{\falsev}{\mathsf{false}}
\newcommand{\Always}{\Box \, }
\newcommand{\Eventually}{\Diamond \, }
\newcommand{\Next}{\bigcirc \, }
\newcommand{\until}{\mbox{$\, {\sf U}\,$}}

\newcommand{\abs}[1]{\lvert#1\rvert}

\newcommand{\calAP}{\mathcal{AP}}

\newcommand{\calG}{\mathcal{G}}

\newcommand{\calZ}{\mathcal{Z}}

\newcommand{\init}{{\iota}}

\renewcommand{\vec}[1]{\mathbf{#1}}

\acrodef{mdp}[MDP]{Markov Decision Process}
\acrodef{pomdp}[POMDP]{Partially Observable Markov Decision Process}
\acrodef{momdp}[MOMDP]{Multi-objective MDP}
\acrodef{ltl}[LTL]{Linear Temporal Logic}
\acrodef{ltlf}[LTL$_f$]{Linear Temporal Logic on Finite Traces}
\acrodef{dfa}[DFA]{Deterministic Finite Automaton}
\acrodef{tlmdp}[TLMDP]{Terminating Labeled Markov Decision Process}
\acrodef{pdfa}[PDFA]{Preference Deterministic Finite Automaton}

\newtheorem{theorem}{Theorem}

  \newtheorem{definition}{Definition}
 \newtheorem{example}{Example}
\newtheorem{problem}{Problem}
\newtheorem{lemma}{Lemma}
\newtheorem{assumption}{Assumption}
\newtheorem{proposition}{Proposition}

\newtheorem{remark}{Remark}

\newcommand{\calA}{\mathcal{A}}

\newcommand{\augnodes}{\mathcal{Y}}
\newcommand{\augnode}{Y}

\newcommand{\augedges}{\mathcal{E}}

\acrodef{smdp}[Semi-MDP]{Semi-Markov decision process}
\acrodef{mcts}[MCTS]{Monte Carlo tree search}
\acrodef{uct}[UCT]{Upper Confidence Bound 1 applied to trees}
\acrodef{scltl}[scLTL]{syntactically co-safe LTL}
\acrodef{ssp}[SSP]{Stochastic Shortest Path}
\acrodef{p2sg}[SG(2)]{Two-player Stochastic Game}
\acrodef{mc}[MC]{Markov chain}
\acrodef{prefltl}[TPL]{ Temporal Preference Logic}
\acrodef{tld}[TLwD]{Temporal Logic with Distributions}
\acrodef{mtl}[Metric TL]{Metric Temporal Logic}
\acrodef{sta}[STA]{Stochastic Timed Automaton}

\newcommand{\dist}{\mathcal{D}}

\renewcommand{\Pr}{\mathbf{Pr}}

\newcommand{\calM}{\mathcal{M}}
\newcommand{\calP}{\mathcal{P}}

 \newcommand{\lang}{\mathcal{L}}

\newcommand{\reach}[1]{\mathsf{reach}(#1)}

\acrodef{gpf}[GPF]{generalized preference formula}

\acrodef{cp}[CP]{ceteris paribus}
\acrodef{milp}[MILP]{Mixed-Integer Linear Programming}
\acrodef{dfa}[DFA]{Deterministic Finite Automaton}

\newcommand{\strictpref}{\triangleright}
\newcommand{\weakpref}{\trianglerighteq}

\newcommand{\prefnode}{W}
\newcommand{\prefnodes}{\mathcal{W}}

\newcommand{\pdfa}{\mathcal{A}}

\newcommand{\prefvertices}{\mathcal{W}}

\newcommand{\prefedges}{E}

\newcommand{\outcomes}{\mathsf{Outcomes}}

\newcommand{\Paths}{\operatorname{Paths}}

\newcommand{\trace}{\operatorname{trace}}

\acrodef{pltlf}[PLTLf]{Preference over linear temporal logic over finite words}

\usepackage{xcolor}
 
\runninghead{Preference-Based Probabilistic Planning}

\title{Preference-Based Planning in Stochastic Environments: From  Partially-Ordered  Temporal Goals to Most Preferred Policies}

\author{Hazhar Rahmani\affilnum{1}, Abhishek N. Kulkarni\affilnum{2} and Jie Fu\affilnum{1}}

\affiliation{\affilnum{1}Department of Electrical and Computer Engineering, University of Florida, USA\\
\affilnum{2}Oden Institute for Computational Engineering and Sciences, University of Texas at Austin, TX, 78712, USA}

\corrauth{Jie Fu, 1889 Museum Road,
5000 Malachowsky Hall,
Gainesville, FL 32611, USA.}

\email{\{h.rahmani, fujie\}@ufl.edu, abhishek.kulkarni@austin.utexas.edu.}

    \begin{abstract}
    Human preferences are not always represented via complete linear orders: It is natural to employ partially-ordered preferences for expressing incomparable outcomes. In this work, we consider decision-making and probabilistic planning in stochastic systems modeled as Markov decision processes (MDPs), given a partially ordered preference over a set of temporally extended goals. Specifically, each temporally extended goal is expressed using a formula in Linear Temporal Logic on Finite Traces (LTL$_f$).  
    To plan with the partially ordered preference, we introduce order theory to map a preference over temporal goals to a preference over policies for the MDP. Accordingly, a most preferred policy under a stochastic ordering induces a stochastic nondominated probability distribution over the finite paths in the MDP. To synthesize a most preferred policy, our technical approach includes two key steps. In the first step, we develop a procedure to transform a partially ordered preference over temporal goals into a computational model, called preference automaton, which is a semi-automaton with a partial order over acceptance conditions. In the second step, we prove that finding a most preferred policy is equivalent to computing a Pareto-optimal policy in a multi-objective MDP that is constructed from the original MDP, the preference automaton, and the chosen stochastic ordering relation. Throughout the paper, we employ running examples to illustrate the proposed preference specification and solution approaches. We demonstrate the efficacy of our algorithm using these examples, providing detailed analysis, and then discuss several potential future directions.	
	\end{abstract}

\begin{document}
 
 \maketitle

\section{Introduction}
	\label{sec:intr}
	With the rise of artificial intelligence and foundational models, 
robotics and other autonomous systems are now designed to understand and respond to human commands in natural language, making human-robot interactions more intuitive and user-friendly. 
However, human commands or preferences are not always expressible by a complete linear order. Preferences may need to admit a \emph{partial} order because of (a) \emph{Inescapability}: An agent has to make decisions under time limits but with partial information about preferences because, for example, it lost communication with the server; and (b) \emph{Incommensurability}: Some situations, for instance, the quality of an apple to that of banana, are fundamentally incomparable since they lack a standard basis for comparison. These situations motivate the need for 
a procedure that translates human preferences into a computational model for autonomous agents and a planner that deals with partially ordered preferences in the presence of all uncertainties in its environment.

In this paper, we  consider preference-based planning (PBP) in stochastic systems modeled as Markov decision processes (MDPs) with user preferences over temporally extended goals. Specifically, we express each temporally extended goal using a formula in Linear Temporal Logic on Finite Traces (LTL$_f$).  
\begin{figure}[h]
		\centering
		\includegraphics[width=1.0\linewidth]{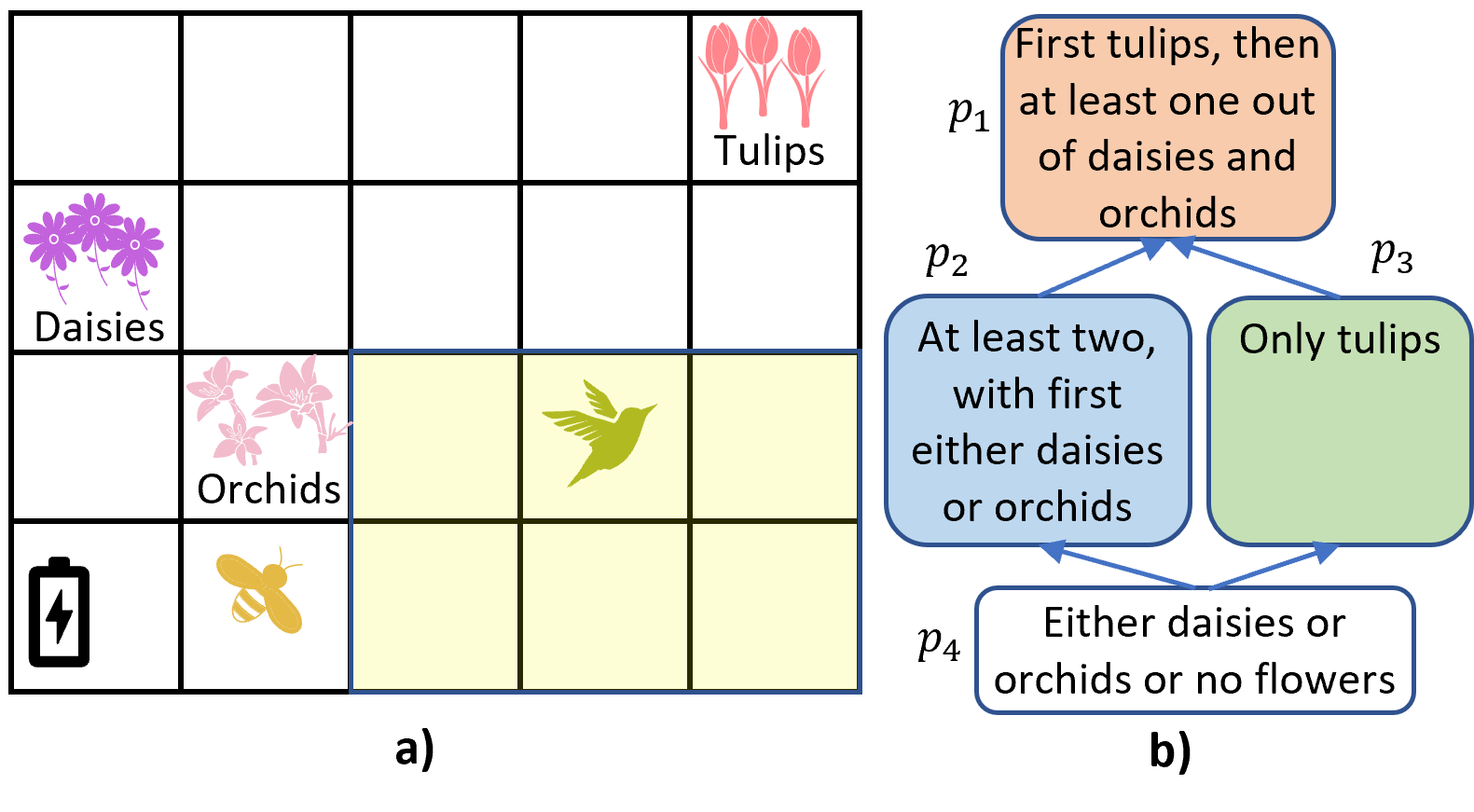}
		\caption{
		 \textbf{a)} Bob's Garden.
		 \textbf{b)} Bob's preferences on how the bee robot should perform the task of pollinating the flowers.
		}
		\label{fig:gap_garden}
	\end{figure}
	For motivation, consider Figure~\ref{fig:gap_garden}, which shows a garden that belongs to Bob. He grows three kinds of flowers: Tulips, daisies, and orchids. To pollinate the flowers, he uses a bee robot with limited battery. The environment is uncertain due to the presence of another agent (bird), the weather, and the robot dynamics.

    Bob has a preference for how the robot should achieve the task of pollination. Compared to the other types, tulips have a shorter life span, so Bob considers four outcomes: 
    \begin{itemize}
        \item[($p_1$)] pollinate tulips first, then at least one other flower type; 
        \item[($p_2$)] pollinate two types of flowers, with the first being either daisies or orchids; 
        \item[($p_3$)] pollinate only tulips; and
        \item[($p_4$)] pollinate either daisies or orchids or no flowers,
    \end{itemize}
    over which, his preference is shown in Figure~\ref{fig:gap_garden}b, using a preference graph, where the nodes represent the outcomes, and each directed edge is an improving flip \citep{santhanamRepresentingReasoningQualitative2016}. According to this graph, $p_1$ is the most preferred outcome, and $p_4$ is the least preferred, while $p_2$ and $p_3$ are incomparable with each other. 
    %
    %
    As the robot has a limited battery life and the system is stochastic, it might not achieve the most preferred outcome with probability one. 
    Thus, we are interested to answer the following questions: Given all the uncertainties in the environment, how to compute a robot's policy that maximally satisfies Bob's preference? How to rank different policies, considering the fact that some outcomes are incomparable?

	Preference-based planning (PBP) enables the system to decide which goals to satisfy when not all of them can be achieved \citep{hastie2010rational}. 
	Even though PBP has been studied since the early 1950's, most works on preference-based temporal planning (c.f. \cite{baier2008planning}) assume that all outcomes are pairwise comparable---that is, the preference relation is a \emph{total} order. 
	This assumption is often strong and, in many cases, unrealistic \citep{aumann1962utility}. 

    With the emergence of large language models translating human commands into temporal logic formulas \citep{chen2023nl2tl,coslerNl2specInteractivelyTranslating2023a}, 
 it becomes natural to consider developing PBP algorithms with human preferences over temporal goals, which are commonly encountered in robotic planning problems. Setting aside natural language understanding, 
  PBP has been well-studied for deterministic systems given both total and partial preferences. See the survey by \cite{baier2008planning}.
  For preferences over temporal goals in deterministic systems, several works \citep{tumova2013least,wongpiromsarn2021,rahmani2020what,rahmani2019optimal} proposed minimum-violation planning methods that decide which low-priority constraints should be violated. 
    \cite{amorese2023optimal} formulate and solve a two-objective optimal planning problem where one objective is to minimize the total cost of a plan, while the other aims to optimize the costs of  individual temporal goals ordered by the user preference.
   
   \cite{mehdipourSpecifyingUserPreferences2021} associate weights with Boolean and temporal operators in signal temporal logic to specify the importance of satisfying the sub-formulas and priority in the timing of satisfaction. This reduces the PBP problem to that of maximizing the weighted satisfaction in deterministic dynamical systems. 
   %
   For planning under this new specification language, \cite{cardona2023mixed} propose an algorithm based on mixed linear integer programming.
    However, the solutions to PBP  for deterministic systems cannot be applied to stochastic systems (such as  MDPs/POMDPs).
    This is because in stochastic systems, even a deterministic policy yields a distribution over outcomes. Hence, to determine a better policy, we need comparison of distributions---a task a deterministic planner cannot do.
    
	Preference-based planning for stochastic systems has been less studied until recently.
	\cite{Lahijanian2016} consider a problem that aims to revise a given specification to improve the probability of satisfaction of the specification. They
	develop an \ac{mdp} planning algorithm that trades off minimizing the cost of revision and maximizing the probability of satisfying the revised formula. 
    \cite{cai2021optimal} focus on planning with infeasible LTL specifications in MDPs. 
    Their problem's objective is to synthesize a policy that, in decreasing order of importance, 1) provides a desired guarantee to satisfy the task, 2) satisfies the specifications as much as possible, and 3) minimizes the implementation cost of the plan.
	\cite{li2020probabilistic} solve a preference-based probabilistic planning problem by reducing it to a multi-objective model checking problem. 
 \cite{liProbabilisticPlanningPrioritized2023} study a class of preferences over temporal goals constructed using prioritized conjunction and ordered disjunction and show that these formulas can be equivalently expressed by weighted automata. They then provide a probabilistic planning algorithm that maximizes the  expected degree of satisfaction. 
	However, all these works assume the preference relation to be \emph{total}. 
	To the best of our knowledge, only \cite{fu2021probabilistic} and \cite{kulkarni2022opportunistic} have studied probabilistic planning with incomplete preferences. 
	\cite{kulkarni2022opportunistic} focus on the qualitative version of the problem, synthesizing strategies that identify and exploit opportunities to improve the most preferred achievable outcome with either positive probability or probability one. This is achieved by reducing the problem to reactive synthesis \citep{manna2012temporal,baier2008principles}. In \cite{fu2021probabilistic}, the author introduced the notion of the value of preference satisfaction for planning within a pre-defined finite time duration and developed a mixed-integer linear program to maximize the satisfaction value for a subset of preference relations. In comparison, our work resorts to the notion of stochastic ordering to rank the policies in the stochastic system with respect to the partial order of temporal goals and allows the time horizon to be finite, but unbounded. 

	Our contributions in this paper are four-fold. (1) We introduce a new computational model called \emph{Preference Deterministic Finite Automaton (PDFA)}, which models a user's (possibly partially-ordered) preference  over temporally extended goals. (2) We introduce an algorithm that translates a set of partially ordered LTL$_f$ formulas, each representing a temporal goal, to a PDFA. (3) We establish a connection between the PBP in stochastic systems and the notions of stochastic orders~\citep{masseyStochasticOrderingsMarkov1987}. This connection allows us to rank policies given their induced probabilistic distribution over possible outcomes. 
    Hence, it 
 reduces probabilistic planning with partially-ordered preferences over temporal goals to computing the set of nondominated policies for a multi-objective \ac{mdp}, constructed as a product of the \ac{mdp} modeling the environment and the PDFA specifying the user preference over the temporal goals.
  (4) We employ the property of weak-stochastic nondominated policies to design multiple objective functions in the product \ac{mdp} and prove that a Pareto-optimal policy in the resulting multi-objective product \ac{mdp} is weak-stochastic nondominated respecting the preference relation. Thus,  
	the set of weak-stochastic nondominated policies can, then, be computed using any off-the-shelf solver that computes Pareto optimal policies.

    The paper is organized as follows.
    In Section~\ref{sec:def}, we present preliminaries and our problem definition.
    In Section~\ref{sec:results}, we introduce \ac{pdfa}, and in Section~\ref{sec:ltlf_2_pdfa}, we present our algorithm for converting a preference model of a set of LTL$_f$ formulas into a \ac{pdfa}.
    In Section~\ref{sec:alg}, we present our algorithm for computing a nondominated a policy, given the \ac{pdfa} specifying the user's preference over temporal goals.
    In Section~\ref{sec:caseStudy}, we present a case study and our detailed analysis.

    We presented a preliminary version of this paper at the 2023 IEEE International Conference 
    on Robotics and Automation~\citep{rahmani2023probabilistic}. 
    In addition to revisions made throughout the paper, we have included several new results: (1) Our preliminary version assumed the \ac{pdfa} is given by the user, but in this version we assume the user's preference is specified using a partially ordered set of LTL$_f$ formulas, and develop an algorithm to translate the partially ordered set of LTL$_f$ formulas into a \ac{pdfa}; (2) the preliminary version considered only the notion of weak-stohastic ordering for comparing policies, but in this version we added two additional notions of stochastic ordering, strong-stochastic ordering and weak$^\ast$-stochastic ordering; and (3) we extended our experiment to include results for the new additional stochastic orderings and discuss how different stochastic orders may affect the policy choices.   

    \section{Definitions}
    \label{sec:def}
	\textbf{Notations}: The set of all finite words over a finite alphabet $\Sigma$ is denoted $\Sigma^\ast$. 
    The empty string, $\Sigma^0$, is denoted as $\epsilon$. 
    %
    We denote the set of all probability distributions over a finite set $X$ by $\dist(X)$.
    Given a distribution $\mathbf{d}\in \dist(X)$, the probability of an outcome $x\in X$ is denoted $\mathbf{d}(x)$. %

	\subsection{The System and its Policy}

    We model the system using a variant of \ac{mdp}.

	\begin{definition}[\ac{tlmdp}]
		\label{def:labeled_mdp}
		A \ac{tlmdp}, or a terminating MDP for short, is a tuple  $M = \langle S, A:=\bigcup_{s \in S} A_s, \mathbf{P}, s_0, s_\bot, \calAP, L \rangle$ in which 
		$S$ is a finite set of states;
		$A$ is a finite set of actions, where for each state $s \in S$, $A_s$ is the set of available actions at $s$;
		$\mathbf{P}: S \times A   \rightarrow \dist(S)$ is the probabilistic transition  function, where for each $s, s' \in S$ and $a \in A$, $\mathbf{P}(s, a, s')$ is the probability that the MDP transitions to $s'$ after taking action $a$ at $s$; 
		$s_0 \in S$ is the initial state;
		$s_\bot \in S$ is the \emph{termination state}, which is a unique \emph{sink} state and $A_{s_\bot} = \emptyset$; 
		$\calAP$ is a finite set of atomic propositions; and
		$L: S \rightarrow 2^{\calAP} \cup \{\epsilon\}$ is a labeling function that assigns to each state $s\in S \setminus \{s_\bot\}$, the set of atomic propositions $L(s) \subseteq \calAP $ that hold in $s$. 
		Only the terminating state is labeled the empty string, i.e., $L(s)=\epsilon$ if and only if $s=s_\bot$.
	\end{definition}

	Though this definition assumes a single sink state, we do not lose generality, as one can always convert an MDP with multiple sink state into an equivalent MDP with a single sink state by  keeping only a single sink state and redirecting all the transitions to other sink states to that sink state. 

	The robot's interaction with the environment in a finite number $k$  of steps produces an \emph{execution} $\varrho = s_0 a_0 s_1 a_1 \cdots s_{k-1} a_{k-1} s_{k}$, where $s_0$ is the initial state and at each step $0 \leq i \leq k$, the system is at state $s_i$, the robot performs $a_i \in A_{s_i}$, and then the system transitions to state $s_{i+1}$, picked randomly based on  the distribution $\mathbf{P}(\cdot \mid s_i, a_i) $.
	%
	This execution produces a \emph{path} defined as $\rho=s_0 s_1 \cdots s_k \in S^*$, 
	and the \emph{trace} of this path is defined as the finite word $\trace(\rho)=L(s_0) L(s_1) L(s_2) \cdots L(s_k)  \in (2^\calAP)^*$.
	%
	%
	Path $\rho$ is called \emph{terminating} if $s_k = s_\bot$.
	The set of all terminating paths in $M$ is denoted $\Paths_{\bot}(M)$.
  A policy for $M$ is a function $\pi: \mathscr{D} \rightarrow \mathscr{C}$ with $\mathscr{D} \in \{S, S^*\}$ and $\mathscr{C} \in \{A, \dist(A)\}$, and it is called \emph{memoryless} if $\mathscr{D}=S$; \emph{finite-memory} if $\mathscr{D}=S^*$; \emph{deterministic} if $\mathscr{C}=A$, and \emph{randomized} if $\mathscr{C}=\dist(A)$.

In a terminating \ac{mdp}, a policy is \emph{proper} if it guarantees that the termination state $s_\bot$ will be reached with probability one \citep{bertsekas1991analysis}. 	
The set of all randomized, finite-memory, proper polices for $M$ is denoted $\Pi_{prop}^M$. 
	%
	%
	In this paper, we consider only the \ac{tlmdp}s for which all the policies are proper. In other words, the system always terminates after a finite time. This restriction is due to that (1) 
many applications require the robot to finish its execution in a finite time, and 2) the preference specification, defined next, is restricted to a partially ordered set of finite traces. 

 \subsection{Specifying the Temporal Goals}
The temporal goals in the MDP are specified formally using the following language.
 
\begin{definition}[Syntax of \ac{ltlf}  \citep{de2013linear}]
Given a finite set $\calAP$ of atomic propositions, a formula in \ac{ltlf} over $\calAP$ is generated by the following grammar:
\[    
\varphi \coloneqq   p \mid \neg \varphi \mid \varphi  \land \varphi  \mid \Next \varphi \mid \varphi  \until \varphi,
\]
where $p \in \calAP$ is an atomic proposition, $\neg$ and $\land$ are the standard Boolean operators negation and conjunction, respectively, and $\Next$ and $\until$ are temporal operators ``Next'' and ``Until'', respectively. 

\end{definition}
The temporal operators are interpreted over sequences of time instants. 
The formula $\Next \varphi$ means at the next time instant, $\varphi$ holds true. Formula $\varphi_1 \until \varphi_2$ holds true at the current time instant if there exists a future time instant at which $\varphi_2$ holds true and for all time instants from the current time until that future time, $\varphi_1$ holds true.
%
An additional temporal operator ``Eventually'' ($\Eventually$) is defined as $\Eventually \varphi := \truev \until \varphi$. 
%
Formula $\Eventually \varphi$ means there exists a future time instant at which $\varphi$ holds true.
The dual of Eventually operator is the ``Always'' ($\Always$). It is defined as $\Always \varphi := \neg \Eventually \neg \varphi$. 
Formula $\Always \varphi$ means $\varphi$ holds true at the current instant and all future instants. 
For formal semantics of \ac{ltlf}, see~\cite{de2013linear}.

\begin{example} \label{ex:ltlf-goals}
  For the example in Figure~\ref{fig:gap_garden}, we set $\calAP = \{o, d, t \}$, in which
 $o$ means orchids are being pollinated, $d$ means daisies are being pollinated, and $t$ means tulips are being pollinated. 
    The temporal goals $p_1$ through $p_4$ in the example in that figure are expressed using the following \ac{ltlf} formulas:
    %
    \begin{itemize}
        \item[($p_1$)] pollinate tulips first, then at least one out of daisies and orchids; 
        \begin{align*}
            (\neg d \wedge \neg o) \until (t \wedge \Next \Eventually ( d \vee o)),
        \end{align*}
        \item[($p_2$)] pollinate two types of flowers, with the first being either daisies or orchids; 
        \begin{align*}
            \neg t \until ((o \wedge \Next \Eventually (d \vee t)) \vee (d \wedge \Next \Eventually ( o \vee  t)) ),
        \end{align*}
        \item[($p_3$)] pollinate only tulips;
        \begin{align*}
            (\neg d \wedge \neg o) \until (t \wedge \Always (\neg d \wedge \neg o)), \text{and}
        \end{align*}
        \item[($p_4$)] pollinate either daisies or orchids or no flowers;
        \begin{align*}
            \Always(\neg d \land \neg o \land \neg t) &\lor (\Eventually o \land \Always(\neg d \land \neg t)) \\ &\lor (\Eventually d \land \Always(\neg o \land \neg t)).
        \end{align*}
    \end{itemize}
    
\end{example}

Given an \ac{ltlf} formula $\varphi$, the words over the alphabet $2^\calAP$ that satisfy $\varphi$, constitute the language of $\varphi$, which is denoted $\lang(\varphi)$. 
%
In the following context, we assume $\Sigma \coloneqq 2^\calAP$.
%


The language of \ac{ltlf} formula $\varphi$ can be represented by the set of words accepted by a finite automaton defined as follows:
\begin{definition}[\ac{dfa}]
 A \ac{dfa} is a tuple $\calA = \langle Q, \Sigma, \delta, q_0, F \rangle$ with a finite set of states $Q$,
    a finite alphabet $\Sigma$, a deterministic transition function $\delta: Q \times \Sigma \rightarrow Q$,
    an initial state $q_0\in Q$, and a set of accepting (final) states $F\subseteq Q$.  For each state $q \in Q$ and letter 
    $\sigma \in \Sigma$,  $\delta(q,\sigma)=q'$ is the state reached upon reading input $\sigma$ from state $q$.
\end{definition}

Slightly abusing the notion, we define the extended transition function $\delta: Q \times \Sigma^* \rightarrow Q$ in the usual manner: 
$\delta(q,\sigma w)=\delta( \delta(q,\sigma),w )$ for a given $\sigma\in \Sigma$ and $w\in \Sigma^\ast$, and $\delta(q, \epsilon) = q$.
A word $w \in \Sigma^\ast$ is \emph{accepted} by the DFA if and only if $\delta(q, w)\in F$. 
The language of $\calA$, denoted $\lang(\calA)$, is set the of all words accepted by the DFA, i.e., $\lang(\calA) = \{w \in \Sigma^* \mid \delta(q, w)\in F \}$.

For each \ac{ltlf} formula $\varphi$, there exists a DFA $\calA_\varphi$ such that $\lang(\varphi) = \lang(\calA_\varphi)$.
Therefore, we can encode each \ac{ltlf} formula using a DFA 
 \citep{de2013linear}.

	\subsection{Rank the Policies}

	We introduce a computational model that captures the user's preference over different temporal goals. 
	\begin{definition}
		\label{def:model_preference}
  A preference model is a tuple $\langle U, \succeq \rangle$ where $U$ is a countable set of outcomes and $\succeq$ is \emph{preorder}--a reflexive and transitive binary relation--on $U$.
	\end{definition}
     Given $u_1, u_2 \in U$, we write $u_1 \succeq u_2$ if $u_1$ is \emph{weakly preferred to} (\ie,  is at least as good as) $u_2$; and $u_1\sim u_2$ if $u_1\succeq u_2$ and $u_2\succeq u_1$,
    that is,  $u_1$ and $u_2$ are \emph{indifferent}.
    We write
    $u_1 \succ u_2$ to mean that $u_1$ is \emph{strictly preferred} to $u_2$, \ie, $u_1\succeq u_2$ and $u_1\not \sim u_2$.  We write
	$u_1 \nparallel u_2$ if $u_1$ and $u_2$ are \emph{incomparable}.

\begin{definition}
Given a preference model $\langle U, \succeq \rangle$,
	let $X$ be  a subset of   $U$, the \emph{upper closure of $X$ }is defined by 
	\[
	X^\uparrow = \{y\mid y \succeq x \text{ for some }x\in X \}, \text{ and }
	\] 
 the \emph{lower closure of $X$ }is defined by
	\[
	X^\downarrow = \{y\mid y \preceq x  \text{ for some }x\in X \}.
	\] 
	
	A set $X$ is called an \emph{increasing set} if $X = X^\uparrow$.
\end{definition}	


\cite{masseyStochasticOrderingsMarkov1987} introduced three different stochastic orderings, called, \emph{strong}, \emph{weak}, and \emph{weak$*$} orderings. The three stochastic orderings differ in how they determine a \emph{family of subsets} of $U$. 

	\begin{definition}\citep{masseyStochasticOrderingsMarkov1987}
		Let $\mathfrak{E}_{st}(U)$, $\mathfrak{E}_{wk}(U)$, and $\mathfrak{E}_{wk*}(U)$ denote the strong, weak, and weak* orderings, respectively. 
		It is defined that
		\[
		\mathfrak{E}_{st}(U) = \{\text{ all increasing sets in $2^U$} \},
		\]
		\[
		\mathfrak{E}_{wk}(U) = \{\{x\}^\uparrow \mid x\in U \} \cup \{U, \emptyset\}, \text{and}
		\]
		\[
		\mathfrak{E}_{wk*}(U) = \{E\setminus \{x\}^\downarrow \mid x\in U \} \cup \{U, \emptyset\}.
		\]
		\end{definition}
	
	These stochastic orderings allow us to rank probability measures according to the partially ordered set $U$.
	Let $\mathfrak{E}(U)$ be \emph{a family of subsets} of $U$ that includes $U$ itself and the empty set $\emptyset$. That is, $\mathfrak{E} \in \{\mathfrak{E}_{st}, \mathfrak{E}_{wk}, \mathfrak{E}_{wk*}\}$. 	
	 Let $P_1$ and $P_2$ be two probability measures on $U$. We denote $P_2\ge_{\mathfrak{E}}P_1$  
	whenever $P_2(X) \ge_{\mathfrak{E}} P_1(X)$ for all subsets $X\in \mathfrak{E}(U)$.
	It is proven that if the partial order $U$ is a total order, then the three stochastic orderings $\mathfrak{E}_{st}(U)$, $\mathfrak{E}_{wk}(U)$, and $\mathfrak{E}_{wk*}(U)$  are equivalent (see Proposition~2.5 of \cite{masseyStochasticOrderingsMarkov1987}). 
	However, for a partial order, the three stochastic orderings may differ.
		%

        To illustrate, consider the following example.
     \begin{example}
	 	Let $U = \{a,b,c,d\}$ and  $\succeq = \{(a, b), (b, d), (c, d), (a, c), (a, d) \} \cup I_U$, where $I_U$ is the identity relation on $U$ and that $(x,y)\in \succeq$ if and only if $x\succeq y$.
       Also, consider probability measures $P_1$, $P_2$, and $P_3$ where $P_1(a)=0.5, P_1(b)=0.3, P_1(c)=0.2$, $P_2(b) = 0.5, P_2(c) = 0.3, P_2(d) = 0.2$, and $P_3(a) = 0.3, P_3(b) = 0.2, P_3(d) = 0.5$.   

        We have
	 		\[
	 	\mathfrak{E}_{st}(U) = \{ \{a\}, \{a,b\}, \{a,c\}, \{a, b, c \}, \{a,b,c,d\}, \emptyset\}.
	 	\]
        Accordingly, 
	 	\[
	 	[P_1[X]]_{X\in	\mathfrak{E}_{st}(U) }= [0.5, 0.8, 0.7, 1, 1, 0], 
	 	\]
	 	\vspace{-18pt}
	 	\[
	    [P_2[X]]_{X\in	\mathfrak{E}_{st}(U) }= [0, 0.5, 0.3, 1, 1, 0], \text{and}
	 	\] 
	 	\vspace{-18pt}
	 	\[
	 	[P_3[X]]_{X\in 	\mathfrak{E}_{st}(U) }   = [0.3, 0.5, 0.3, 0.5, 1, 0].
	 	\]
   Therefore, $P_1 >_{\mathfrak{E}_{st}} P_2$, $P_1 >_{\mathfrak{E}_{st}} P_3$. None of $P_2$ and $P_3$ strong-stochastic dominates the other one.
   
	 	Also, we have
	 		\[
	 	\mathfrak{E}_{wk}(U) = \{ \{a\}, \{a,b\}, \{a,c\}, \{a,b,c,d\}, \emptyset\}.
	 	\]
	 	Accordingly, 
	 	\[
	 	[P_1[X]]_{X\in	\mathfrak{E}_{wk}(U) }= [0.5, 0.8, 0.7, 1, 0],
	 	\]
	 	\vspace{-18pt}
	 	\[
	    [P_2[X]]_{X\in	\mathfrak{E}_{wk}(U) }= [0, 0.5, 0.3, 1, 0], \text{and}
	 	\] 
	 	\vspace{-18pt}
	 	\[
	 	[P_3[X]]_{X\in 	\mathfrak{E}_{wk}(U) }   = [0.3, 0.5, 0.3, 1, 0].
	 	\]
	 	Thus, $P_1 >_{\mathfrak{E}_{wk}} P_2$, $P_1 >_{\mathfrak{E}_{wk}} P_3$, and $P_3 >_{\mathfrak{E}_{wk}} P_2$.

        Also, we have
	 		\[
	 	\mathfrak{E}_{wk*}(U) = \{ \{a,b\}, \{a,c\}, \{a, b, c \}, \{a,b,c,d\}, \emptyset\}.
	 	\]
	 	Accordingly, 
	 	\[
	 	[P_1[X]]_{X\in	\mathfrak{E}_{wk*}(U) }= [0.8, 0.7, 1, 1, 0],
	 	\]
	 	\vspace{-18pt}
	 	\[
	    [P_2[X]]_{X\in	\mathfrak{E}_{wk*}(U) }= [0.5, 0.3, 0.8, 1, 0], \text{and}
	 	\] 
	 	\vspace{-18pt}
	 	\[
	 	[P_3[X]]_{X\in 	\mathfrak{E}_{wk*}(U) }   = [0.5, 0.3, 0.5, 1, 0].
	 	\]
	 	Thus, $P_1 >_{\mathfrak{E}_{wk*}} P_2$, $P_1 >_{\mathfrak{E}_{wk*}} P_3$, and $P_2 >_{\mathfrak{E}_{wk*}} P_3$.
    \end{example}
 In our context, because we are interested in sequential decision-making and planning problems with a finite time termination, the set $U$ is selected to be $\Sigma^\ast$, or the set of finite traces generated by the system and its labeling function.
	%
 Based on the ranking of probability measures induced by each one of the stochastic orderings for $\Sigma^\ast$, we can rank the proper policies $\Pi^M_{prop}$ in the \ac{tlmdp} as follows.
	 
	Note that a proper policy $\pi: S^\ast \rightarrow \dist(A)$ produces a distribution 
	over the set of all terminating paths in the MDP $M$ such that for each terminating path $\rho \in \Paths_{\bot}(M)$,
	$\Pr^\pi(\rho)$
	is the probability of generating $\rho$ when the robot uses policy $\pi$.
	    Each terminating path $\rho$ is mapped to a single word in $\Sigma^*$, namely $\trace(\rho) = L(s_0) L(s_1) \ldots$, 
		and therefore, $\pi$ yields a distribution 
		over the set of all finite words over $\Sigma$
		such that for each word $w \in \Sigma^\ast$, $\Pr^\pi(w)$ is the probability that 
		$\pi$ produces $w$. Formally,
  \[
  \Pr^\pi(w) = \sum_{\rho \in \Paths_{\bot}(M): L(\rho)=w} \Pr^\pi(\rho).
  \]
    Additionally, for a subset $X \subseteq \Sigma^*$, $\Pr^\pi(X)$ is the probability of the words  generated by $\pi$ to be within $X$. Formally,
    \[
  \Pr^\pi(X) = \sum_{w \in X} \Pr^\pi(w).
  \] 
	
		
	

%
	
  %
%
%
%
\begin{definition}
\label{def:weak_dominating_policies}
        Let $\mathfrak{E} \in \{\mathfrak{E}_{st}, \mathfrak{E}_{wk}, \mathfrak{E}_{wk*}\}$ be a stochastic ordering and $\langle U \coloneqq \Sigma^\ast, \succeq\rangle$ be a preference model.
         Given two proper policies $\pi$ and $\pi'$ for the terminating labeled \ac{mdp} $M$, $\pi$ \emph{$\mathfrak{E}$-stochastic dominates} $\pi'$, denoted $\pi >_{\mathfrak{E}} \pi'$, if for each subset $X \in \mathfrak{E}(U)$, it holds
         that $\Pr^{\pi}(X) \ge \Pr^{\pi'}(X)$, and there exists a subset $Y \in \mathfrak{E}(U)$ such that
         $\Pr^{\pi}(Y) > \Pr^{\pi'}(Y)$.
\end{definition}

%

This definition is used to introduce the following notion.
	\begin{definition}
		\label{def:sto_dominance}
		   A proper policy $\pi \in \Pi_{prop}^M$ is \emph{$\mathfrak{E}$-stochastic nondominated} if there \emph{does not exist} any policy $\pi' \in \Pi_{prop}^M$ such that $\pi' >_{\mathfrak{E}} \pi$.
	\end{definition}
	
Informally, we say a policy $\pi$ is \emph{$\mathfrak{E}$-preferred} if and only if it is $\mathfrak{E}$-stochastic nondominated in $\Pi_{prop}^M$.
    

We aim to solve the following planning problem:
	\begin{problem}
 \label{problem:main-problem}
	    Given a terminating labeled \ac{mdp} $M = \langle S, A:=\bigcup_{s \in S} A_s, \mathbf{P}, s_0, s_\bot, \calAP, L \rangle$, a preference model $\langle \Sigma^\ast, \succeq \rangle$, and a stochastic ordering $\mathfrak{E} \in \{\mathfrak{E}_{st}, \mathfrak{E}_{wk}, \mathfrak{E}_{wk*}\}$, compute a proper policy that is $\mathfrak{E}-$stochastic nondominated.
	\end{problem}
		%

	\section{Modeling Preference over LTL$_f$ Goals}
    \label{sec:results}
	
	In this section, we consider the case when the user defines their preference over temporal goals. 
	


 The user specifies the temporal goals using \ac{ltlf} formulas, one formula for each goal, and then expresses their preference over these goals using a
    preference model over the set of these formulas. 
	
	\begin{definition}
	\label{def:preference-model-ltl}
  An \emph{\ac{ltlf} preference model} is a preference model $\langle \Phi, \weakpref \rangle$ in which $\Phi = \{\varphi_1,\ldots, \varphi_N \}$ is a finite set of  \emph{distinct} \ac{ltlf} formulas over a set of atomic propositions $\calAP$ and $\weakpref$ is a partial order---a reflexive, transitive, and an antisymmetric---relation on $\Phi$.
\end{definition}

Two \ac{ltlf} formulas $\varphi$ and $\varphi'$ are distinct if $\mathcal{L}(\varphi) \ne \mathcal{L}(\varphi')$, where $\mathcal{L}(\varphi)$ is the language of the formula, \ie, the set of words satisfying the formula. 


\begin{assumption}
	\label{assume:cover}
    We assume that $\bigcup_{1 \leq i \leq N} \mathcal{L}(\varphi_i) = \Sigma^\ast$, meaning that for each word $w \in \Sigma^\ast$, there
    is at least one $\varphi \in \Phi$ such that $w \in \mathcal{L}(\varphi)$.
    Note that if $\bigcup_{1 \leq i \leq N} \mathcal{L}(\varphi_i) \subset \Sigma^\ast$, then the assumption will hold by adding the formula $\varphi = \bigwedge_{1 \leq i \leq N} \neg \varphi_i$ to $\Phi$.
\end{assumption}

\begin{assumption}
	\label{assume:partial-order}
	The preference model $\langle \Phi, \weakpref \rangle$ is a partial order relation over $\Phi$, which means the following properties are satisfied:
	\begin{itemize}
		\item Reflexive: $\varphi \weakpref \varphi$ for all $\varphi\in \Phi$.
		\item Antisymmetric: $\varphi \weakpref \varphi'$ and $\varphi'\weakpref \varphi$ implies $\varphi=\varphi'$.
		\item Transitive: $\varphi_1\weakpref  \varphi_2$ and $\varphi_2 \weakpref \varphi_3$ implies $\varphi_1\weakpref \varphi_3$.
	\end{itemize}
\end{assumption}	

\begin{remark}
Note that any preference model $\langle \Phi, \weakpref \rangle$ in which $\weakpref$ is a preorder, \ie, a partial order without the requirement of being antisymmetric, can be converted into a preference model $\langle \Phi', \weakpref' \rangle$ in which $\weakpref'$ is a partial order. 
%
The idea is to iteratively refine the preference model until the resulting model has no pair of indifferent formulas.
In each step, two indifferent formulas $\varphi_i$ and $\varphi_j$ are replaced by their disjunction $\varphi_i \lor \varphi_j$, after which the preference relation is altered accordingly.  
\end{remark}

%
%

	The model $\langle \Phi, \weakpref \rangle$ is a combinative preference model, as opposed to an exclusionary one. This is because we do not assert the exclusivity condition $ \varphi_i\land \varphi_j = \falsev$. This allows us to represent a preference such as ``Visiting A and B is preferred to visiting A,'' ($\Eventually A\land \Eventually B\weakpref \Eventually A$) where the less preferred outcome must be satisfied first in order to satisfy the more preferred outcome. In literature, it is common to study exclusionary preference models (see \cite{baier2008planning,bienvenu2011specifying} and the references within) because of their simplicity \cite{stanford2022preferences}.  
 %
 However, we focus on planning with combinative preferences since they are more expressive than the exclusionary ones \citep{hansson2001structure}.  In fact, every exclusionary preference model can be transformed into a combinative one, but the opposite is not true. 
	
	%

	When a combinative preference model is interpreted over finite words, the agent needs a way to compare the sets of temporal logic objectives satisfied by two words. 
	For instance, consider the preference that ``Visiting A and B is preferred to visiting A,'' and let $w_1=\emptyset \{A\}\emptyset \{B\}$ and $w_2 = \emptyset \{A\} \emptyset$ be two finite
	 words. 
  Note that $w_1$ has both $A$ and $B$ evaluated true, each at some point in time, and $w_2 = \emptyset \{A\} \emptyset $ only has $A$ evaluated true. 
  Therefore, $w_1\models \varphi_1 \land \varphi_2$, whereas $w_2$ satisfies only $\varphi_2$. To determine the preference between $w_1$ and $w_2$, the agent compares the set $\{\varphi_1,\varphi_2\}$ with $\{\varphi_2\}$ to conclude that $w_1$ is preferred over $w_2$. However, suppose the given preference is that ``visiting A is preferred over visiting B,'' i.e., ($\Eventually A\weakpref \Eventually B$). Then the two words $w_1$ and $w_2$ are indifferent since both satisfy the more preferred objective $\Eventually A$.  To formalize this notion, we define the notion of most-preferred outcomes.

	Given a non-empty subset $\mathbb{X} \subseteq \Phi$, let $\mathsf{MP}(\mathbb{X}) \triangleq \{\varphi \in \mathbb{X} \mid \nexists \varphi' \in \mathbb{X}: \varphi' \weakpref \varphi\}$ denote the set of most-preferred outcomes in $\mathbb{X}$.
	
	\begin{definition}
	\label{def:mp-outcomes}
	Given an \ac{ltlf} preference model $\langle \Phi, \weakpref \rangle$ and a finite word $w  \in \Sigma^\ast$, the set of most-preferred formulas satisfied by $w$ is given by $\mathsf{MP}(w) \coloneqq \mathsf{MP}(\{ \varphi \in \Phi \mid  w\models \varphi\})$. 
	\end{definition}

	By definition, there is no outcome included in $\mathsf{MP}(w)$ that is weakly preferred to any other outcome in $\mathsf{MP}(w)$. Thus, we have the following result.
	
	\begin{lemma}
 \label{lma:incomparable-in-MP}
	For any word $w \in \Sigma^\ast$, formulas in $\mathsf{MP}(w)$ are incomparable to each other. 
	\end{lemma}
\begin{proof}
    By contradiction. Suppose that the set $\mathsf{MP}(w)$ contains two formulas $\varphi_1$ and $\varphi_2$ that are comparable. Then, since $\weakpref$ is a partial order, one of the following cases must be true:  
    \begin{inparaenum}[1)]
        \item $\varphi_1 \strictpref \varphi_2$, or
        \item $\varphi_2 \strictpref \varphi_1$.
    \end{inparaenum} 
    Consider the first case. By definition of $\mathsf{MP}$ operator, only $\varphi_1$ is included in $\mathsf{MP}(w)$. Similarly, in second case, only $\varphi_2$ is included in $\mathsf{MP}(w)$. This is a contradiction. 
\end{proof}

	Now, we formally define the interpretation of $\langle \Phi, \weakpref \rangle$ in terms of the preference relation it induces on $\Sigma^\ast$.

	\begin{definition}
		\label{def:preference-ltl}
	An \ac{ltlf} preference model $\langle \Phi, \weakpref \rangle$ induces the preference model $\langle \Sigma^\ast, \succeq \rangle$ where for any $w_1, w_2 \in \Sigma^\ast$, 
	
	\begin{itemize}
	\item $w_1 \succeq w_2$ if and only if  for every formula $\varphi \in \mathsf{MP}(w_1)$, there exists a formula  $\varphi'\in \mathsf{MP}(w_2)$ such that $\varphi \weakpref \varphi'$,  
	
	\item $w_1 \sim w_2$ if and only if $\mathsf{MP}(w_1) = \mathsf{MP}(w_2)$, and
	
	\item $w_1 \nparallel w_2$, otherwise.
	\end{itemize}
	\end{definition}



The following set of properties can be shown.

\begin{lemma}
	Letting $\langle \Sigma^\ast, \succeq \rangle$ be the preference model induced by $\langle \Phi, \weakpref \rangle$, for any $w_1, w_2\in \Sigma^\ast$, if $w_1\succeq w_2$, then   
	there does not exist a pair of outcomes $\varphi \in \mathsf{MP}(w_1)$ and $\varphi' \in \mathsf{MP}(w_2)$ such that $\varphi' \weakpref \varphi$.
\end{lemma}
\begin{proof}
By contradiction.
	Let $ \mathsf{MP}(w_1) = \{\varphi_1,\ldots, \varphi_m\}  $ and $\mathsf{MP}(w_2) = \{\psi_1,\ldots, \psi_t\}$. 
 Suppose there exists $\psi \in \mathsf{MP}(w_2)$ such that $\psi \weakpref \varphi$ for some $\varphi \in \mathsf{MP}(w_1)$. 
 Given the assumption $w_1 \succeq w_2$, by Definition~\ref{def:preference-ltl}, there exists $\psi' \in \mathsf{MP}(w_2)$ such that $\varphi \weakpref \psi'$. 
 As a result, $\psi \weakpref \varphi \weakpref \psi'$, implying that $\psi \weakpref \psi'$. This contradicts the result in Lemma~\ref{lma:incomparable-in-MP} which imposes $\mathsf{MP}(w_2)$ to contain only incomparable formulas. 
 Thus, the assumption that $\psi \weakpref \varphi$ is contradicted.
\end{proof} 

\begin{lemma}
If $w_1 \sim w_2$, then $\mathsf{MP}(w_1)= \mathsf{MP}(w_2)$.
\end{lemma}
\begin{proof}
By way of contradiction, suppose $\mathsf{MP}(w_1)\ne \mathsf{MP}(w_2)$. 
Without loss of generality, let $\varphi \in \mathsf{MP}(w_1) \setminus \mathsf{MP}(w_2)$. 
Given $w_1\succeq w_2$ and $\varphi \in \mathsf{MP}(w_1)$, there must exist a formula $\psi\in \mathsf{MP}(w_2)$ such that $\varphi \weakpref \psi$. 
Also, because $w_2\succeq w_1$, there exists a formula $\phi \in \mathsf{MP}(w_1)$ such that $\psi \weakpref \phi$. 
Due to the transitivity of $\weakpref$, $\varphi \weakpref \psi \weakpref \phi$ and thus $\varphi \weakpref \phi$ and that $\varphi, \phi \in \mathsf{MP}(w_1)$, contradicting Lemma~\ref{lma:incomparable-in-MP}. 
Since $\varphi$ is chosen arbitrarily, witnessing this contradiction implies that $\mathsf{MP}(w_1)= \mathsf{MP}(w_2)$.
\end{proof}

 
	\begin{lemma}
		The preference model $\langle \Sigma^\ast, \succeq \rangle$  induced by $\langle \Phi, \weakpref \rangle$ is a   preorder. 
	\end{lemma}
\begin{proof}
   For any $w\in \Sigma^\ast$, $w \succeq w$ because for any $\varphi\in \mathsf{MP}(w)$, $\varphi\weakpref \varphi$.
 	Thus, $\langle \Sigma^\ast, \succeq \rangle$ is reflexive. 
    %
For the transitivity, supposing $w_1\weakpref w_2$ and $w_2\weakpref w_3$, we need to show that $w_1\weakpref w_3$. 
Let for each $t \in \{1, 2, 3 \}$, $\mathsf{MP}(w_t)=\{\varphi_{t, i}\mid i=1,\ldots, n_t \}$ be the most prefered formulas satisfied by $w_t$.
%
%
Given that $w_1\succeq w_2$, for any $\varphi_{1,i} \in \mathsf{MP}(w_1)$, there exists $\varphi_{2,j} \in \mathsf{MP}(w_2)$ such that $\varphi_{1,i} \weakpref \varphi_{2,j}$. 
Also, because $w_3\succeq w_3$, for any such $\varphi_{2, j}$, there exists $\varphi_{3, k} \in \mathsf{MP}(w_3)$
such that $\varphi_{2,j}\weakpref \varphi_{3,k}$. 
Using the transitivity property of $\weakpref$, $\varphi_{1, i}\weakpref \varphi_{3, k}$ holds. 
As a result, $w_1\succeq w_3$.
\end{proof}
 Note that the preference relation $\langle \Sigma^\ast, \succeq \rangle$ needs not to be antisymmetric since there might exist two words $w_1\neq w_2$ such that $\mathsf{MP}(w_1) = \mathsf{MP}(w_2)$.
For example, with $\Phi = \{\Eventually a, \Eventually b \}$, consider two words $\emptyset \emptyset \{a\} \{b\}$ and $\emptyset \{a\} \emptyset \{b\}$.  
Since they both satisfy both $\Eventually a$ and $\Eventually b$, $w_1 \succeq w_2$ and $w_2 \succeq w_1$, while $w_1 \neq w_2$, showing an example where $\succeq$ is not antisymmetric.
 

\section{Constructing a Computational Model for an LTL$_f$ Preference Model}
\label{sec:ltlf_2_pdfa}
	
	In this section, we  introduce a novel computational model called a \acf{pdfa}, which encodes the preference model $\langle \Sigma^\ast, \succeq \rangle$ into an automaton. We present a procedure to construct a \ac{pdfa} for a given preference model $\calP = \langle \Phi, \weakpref \rangle$ and prove its correctness with respect to the interpretation in Definition~\ref{def:preference-ltl}.


 	\begin{algorithm}[t]
		\caption{Construction of Preference Graph}
		\label{alg:pref-graph}
		\begin{algorithmic}[1]
		\Function{PrefGraph}{$\langle \Phi, \weakpref \rangle, \langle Q, \Sigma, \delta, \init \rangle$}
			\State Initialize $\prefnodes = \emptyset, \prefedges = \emptyset$, $Z =\emptyset$.
			\ForAll{$(\vec{q}, \vec{q}') \in Q\times Q$} \label{line:for} \label{line:outer_loop}
				\If{$\mathsf{MP}(\vec{q})= \mathsf{MP}(\vec{q}') $}
					\State $Z= Z\cup \{(\vec{q},\vec{q}'), (\vec{q}',\vec{q})\}$
				\Else
                    \State Initialize $D$ as an empty set of sets.
                    \ForAll{$\varphi \in \mathsf{MP}(\vec{q})$}
                        \State Add $\{\varphi' \in \mathsf{MP}(\vec{q}') \mid \varphi \weakpref \varphi'\}$ to $D$.
                    \EndFor
                    \If {$\emptyset \notin D$}
                        \State $Z \gets Z \cup \{(\vec{q}', \vec{q})\}$. \label{line:z}
                    \EndIf
                \EndIf
			\EndFor
			\State $\prefnodes \gets \mathsf{getSCC}(\langle Q, Z \rangle)$ \label{line:SCC}
			\ForAll{$\prefnode, \prefnode' \in \prefnodes$} \label{line:begin_2_loop}
				\If{$\exists \vec{q}\in \prefnode, \vec{q}'\in \prefnode': (\vec{q}, \vec{q}') \in Z$}
					\State $\prefedges\gets \prefedges \cup \{ (\prefnode, \prefnode') \}$ \label{line:end_2_loop}
				\EndIf
			\EndFor  
			\State \Return $G = \langle \prefnodes, \prefedges \rangle$
		\EndFunction 
		\end{algorithmic}
	\end{algorithm}

    	\begin{definition}
		\label{def:pdfa}
		A \ac{pdfa} for an alphabet $\Sigma$ is a tuple $\pdfa= \langle Q, \Sigma, \delta, \init, G:=(\prefnodes, \prefedges) \rangle$
		in which $Q$ is a finite set of states;
		$\Sigma$ is the alphabet;
		$\delta: Q \times \Sigma \rightarrow Q$ is a transition function;
		$\init \in Q$ is the initial state; 
		and $G=(\prefvertices, \prefedges)$ is a \emph{preference graph} in which,
		$\prefvertices = \lbrace W_1, W_2, \cdots, W_m \rbrace$ is a partition of $Q$---i.e., 
		    $W \subseteq Q$ for each $W \in \prefvertices$, $W \cap W' = \emptyset$ for each distinct state subsets
		     $W, W' \in \prefvertices$, and $\bigcup_{W \in \prefvertices} W = Q$; and $\prefedges \subseteq \prefvertices \times \prefvertices$ is a set of directed edges.
		\end{definition}  

 

 With a slight abuse of notation, we define the extended transition function $\delta: Q \times \Sigma^\ast \rightarrow Q$ in the usual way, \ie, $\delta(q,\sigma w) = \delta(\delta(q,\sigma), w)$ for $w\in \Sigma^\ast$ and $\sigma \in \Sigma$, and $\delta(q, \epsilon)=q$.
		%
		Note that Definition~\ref{def:pdfa} augments 
        a \ac{dfa}
  with the preference graph $G$, instead of a set of accepting (final) states. 

		For two vertices $W, W' \in \prefvertices$, we write $W  \rightsquigarrow W'$ to denote $W'$ is \emph{reachable} from $W$.
		By definition, 
  each vertex $W$ of $G$ is reachable from itself. That is, $W \rightsquigarrow W$ always holds.

		The \ac{pdfa} encodes a preference model $\succeq$ for $\Sigma^*=(2^\calAP)^*$ as follows. 
		Consider two words $w, w' \in \Sigma^*$.
		Let $W, W' \in \prefvertices$ be the two state subsets such that $\delta(q, w) \in W $ and $\delta(q, w') \in W'$ (recall that $\prefvertices$ is a partitioning of $Q$);
		There are four cases: (1) if $W  = W'$, then $w \sim w'$; (2) if $W  \neq W'$ and $W' \rightsquigarrow W $, then $w \succ w'$; (3) if $W  \neq W' $ and $W \rightsquigarrow W'$, then $w' \succ w$; and (4) otherwise, $w \nparallel w'$.

     For an example, see Figure~\ref{fig:pdfa}, which shows the \ac{pdfa} for the preferences $p_1$ through $p_4$ in Figure~\ref{fig:gap_garden}.

	  Next, we describe the construction of \ac{pdfa} given a preference model $\calP = \langle \Phi, \weakpref \rangle$. 
	  The construction involves two steps, namely, the construction of the underlying graph of \ac{pdfa} and the construction of the preference graph.

	\begin{definition}
		\label{def:product-graph}
		Given a preference model $\langle  \Phi, \weakpref \rangle $,
		for each formula $\varphi_i \in \Phi$, let $\calA_i = \langle Q_i, \Sigma, \delta_i, \init_i, F_i\rangle$ be the   \ac{dfa} representing the language of $\varphi_i$.
        %
        The underlying automaton of the \ac{pdfa} representing $\langle  \Phi, \weakpref \rangle $ is the tuple, 
		\[
		\langle Q, \Sigma , \delta, \init \rangle
		\]
		in which $Q = Q_1 \times Q_2 \times \cdots Q_n$ is the set of states in \ac{pdfa};
  $\delta: Q \times \Sigma \rightarrow Q$ is a deterministic transition function where for each $\vec{q}=(q_1, q_2, \cdots, q_n) \in Q$ and $a \in \Sigma$, $\delta(\vec{q}, a) = (\delta_1(q_1, a), \delta_2(q_2, a), \cdots, \delta(q_n, a))$; and $\vec{\init} = (\init_1, \init_2, \cdots, \init_n)$ is the initial state.
  \end{definition}

  We represent each state in $Q$ as a vector $\vec{q}$ and the $i$-th component of $\vec{q}$, denoted as $\vec{q}[i]$, is the state in $Q_i$.

  Algorithm~\ref{alg:pref-graph} describes a procedure to construct the preference graph.  It uses the following definition that slightly abuses the notation $\mathsf{MP}(\cdot)$: For each product state $\vec{q}$, we define the set \[ \mathsf{MP}(\vec{q}) = \mathsf{MP}( \left\{ \varphi_i \in \Phi  \mid \vec{q}[i] \in F_i	\right\} )
	\]
	In words, $\mathsf{MP}(\vec{q}) $ is a set of most preferred formulas satisfied by any word that ends in $\vec{q}$.

	 Given the preference model $\langle \Phi,\weakpref \rangle$ and the underlying automaton $\langle Q, \Sigma , \delta, \init \rangle$ of the \ac{pdfa}, lines \texttt{\ref{line:for}-\ref{line:z}} of Algorithm~\ref{alg:pref-graph} construct a set $Z$ of directed edges such that $(\vec{q}',\vec{q})\in Z$ if 
	 and only if for every $\varphi \in \mathsf{MP}(\vec{q})$, there exists a formula $\varphi' \in \mathsf{MP}(\vec{q}')$  such that $\varphi \weakpref \varphi'$. 
    %
  \begin{figure}[h]
		\centering
		\includegraphics[width=1.0\linewidth]{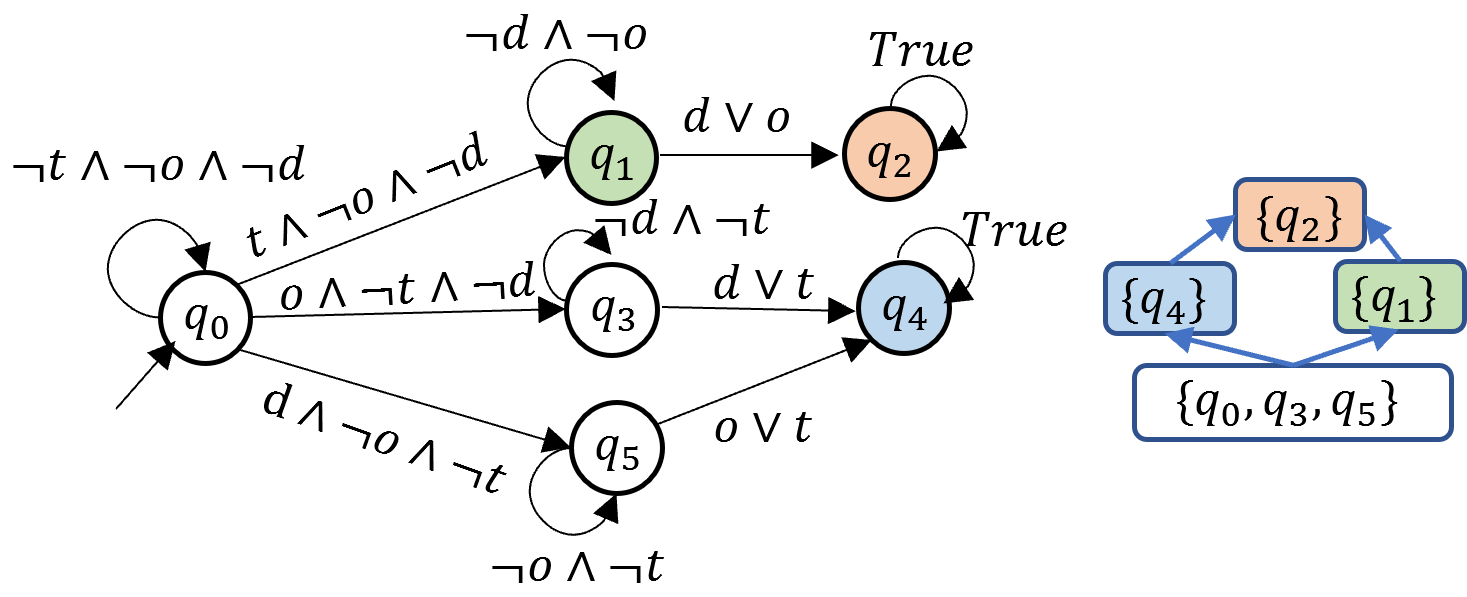}
		\caption{
		 PDFA for the example in Figure~\ref{fig:gap_garden}. 
		 \textbf{Left)} The DFA structure of PDFA.
		 \textbf{Right)} The preference graph of PDFA.
		}
		\label{fig:pdfa}
	\end{figure}
	 Lines \texttt{\ref{line:SCC}-\ref{line:end_2_loop}} of Algorithm~\ref{alg:pref-graph} shows how the set $\prefnodes$ and edges $\prefedges$ of the preference graph are constructed. Using the set of directed edges $Z$, the algorithm computes first the set $\prefnodes$ as the set of strongly connected components of the graph given by state set $Q$ and edges $Z$. Then, a directed edge from an SCC $W$ to another SCC $W'$ is added if there is a state $\vec{q}\in W$ and a state $\vec{q}\in W'$ such that $(\vec{q},\vec{q'})\in Z$.

  In Section~\ref{sec:caseStudy}, we provide a detailed explanation of the construction of the PDFA for the preferences in Figure~\ref{fig:gap_garden}, implemented through our algorithm.

  We next show how the \ac{pdfa} constructed using the product operation in Definition ~\ref{def:product-graph} and Algorithm ~\ref{alg:pref-graph} encodes the exact preference model $\langle \Phi, \weakpref\rangle $.

	\begin{proposition}
		\label{prop:node-property}
		Let $\prefnodes$ be the set of nodes constructed by Algorithm~\ref{alg:pref-graph}. 
        For each $\vec{q} , \vec{q'} \in Q$ for which $\mathsf{MP}(\vec{q}) =  \mathsf{MP}(\vec{q}')$, it holds that $\vec{q}$ and $\vec{q}'$ are included in the same node in $\prefnodes$.
	\end{proposition}
	\begin{proof}
		By the construction in Line~5 of Algorithm~\ref{alg:pref-graph} and the definition of strongly connected components \citep{cormen2022introduction}.
	\end{proof}

\begin{proposition}
	\label{prop:edge-property}
	If  $(\vec{q} , \vec{q'})\not \in Z$, then in graph $\langle Q, Z \rangle$ there is no directed path from $\vec{q}$ to $\vec{q'}$.
\end{proposition}
\begin{proof}
	For the sake of contradiction, suppose $(\vec{q},\vec{q}')\notin Z$ but  there is a directed path of length greater than $ 1$ from $\vec{q}$ to $\vec{q}'$. Let this path be $\vec{q}_0 \rightarrow \vec{q}_1 \rightarrow \cdots   \rightarrow \vec{q}_n$ where $\vec{q}_0=\vec{q}$, $\vec{q}_n=\vec{q}'$, and  $\vec{q}_1$ through $\vec{q}_{n-1}$ are intermediate states along the path.
For $i=0,\ldots, n$,  let $X_i = \mathsf{MP}(\vec{q}_i)$. By the construction, for any $\varphi \in X_{i+1}$, there exists a formula $\psi \in X_{i}$ such that $\varphi \weakpref \psi$. Applying the transitivity of the preference $\weakpref$, it holds that for any $\varphi\in X_n = \mathsf{MP}(\vec{q}')$, there exists a formula $\psi \in X_0 = \mathsf{MP}(\vec{q})$ such that $\varphi \weakpref \psi$. 
As a result, $(\vec{q},\vec{q}')\in Z$, contradicting the assumption.
	
\end{proof}
	
	\begin{proposition}
		\label{prop:partition}
		Set $\prefnodes$ constructed by Algorithm~\ref{alg:pref-graph} partitions $Q$.
	\end{proposition}
	\begin{proof}
		This property automatically holds due to the property of strongly connected components \citep{cormen2022introduction}.
	\end{proof}

\begin{theorem}
	Let $\langle \Sigma^\ast, \succeq \rangle $ be the preference model induced by the semantics of $\langle \Phi, \weakpref \rangle $ (Definition ~\ref{def:preference-ltl}). Given the \ac{pdfa}  $\mathcal{A} = \langle  Q, \Sigma, \delta, \init, G \rangle $ constructed for the preference model $\langle  \Phi, \weakpref \rangle$ using Definition~\ref{def:product-graph} and Algorithm~\ref{alg:pref-graph}, for any $w,w'\in \Sigma^\ast$ let $W, W' \in \prefnodes$ be the nodes such that $\delta(\init, w) \in \prefnode$ and $\delta(\init, w') \in \prefnode'$, the following statements are established:
	\begin{itemize}
	
		\item (Case 1)  
		  $ \prefnode  = \prefnode'$ if and only if
		$w'\sim  w$.
			\item (Case 2)  $\prefnode \ne \prefnode'$ and 
		$\prefnode \rightarrow \prefnode' $ if and only if $w' \succeq w $ and $w\not\sim w'$.

\item (Case 3)  $\prefnode \ne \prefnode'$ and 
		$\prefnode' \rightarrow \prefnode $ if and only if $w \succeq w' $ and $w\not\sim w'$.
  
		\item (Case 4) $w\nparallel w'$, otherwise.
	\end{itemize} 
\end{theorem} 
	\begin{proof}
		Let $\vec{q} = \delta(\init, w)$ and $\vec{q}' =\delta(\init, w')$.
		By construction of the function $\mathsf{MP}(\cdot)$ and the product operation in Definition~\ref{def:product-graph}, the following equation holds:
		\begin{align*}
		\mathsf{MP}(w) 
		&= \mathsf{MP}(\{\varphi_i\mid \delta_i(\init_i, w) \in F_i \}) \\
		&=  \mathsf{MP}(\{\varphi_i\mid \vec{q}[i] \in F_i \}) \\
		&=  \mathsf{MP}(\vec{q}) 
		\end{align*}

Case 1:  ($\Rightarrow$)  If $\prefnode=\prefnode'$, then both $\vec{q}, \vec{q}'\in \prefnode$. This means that $\vec{q}\rightsquigarrow \vec{q}'$ and $\vec{q}' \rightsquigarrow \vec{q}$.  By proposition~\ref{prop:edge-property}, it is only possible that $(\vec{q},\vec{q}')\in Z$ and $(\vec{q}',\vec{q})\in Z$. $\mathsf{MP}(\vec{q}) =  \mathsf{MP}(\vec{q}')$ can be derived due to the antisymmetric property in the partial order of $\langle \Phi, \weakpref \rangle$.

($\Leftarrow$)   If $w\sim w'$, then $\mathsf{MP}(\vec{q})  = \mathsf{MP}(\vec{q}')$ and therefore $\prefnode=\prefnode'$.

Case 2: ($\Rightarrow$) If $\prefnode \rightarrow \prefnode'$, 
then given the construction of the preference graph by lines~\ref{line:begin_2_loop}-\ref{line:end_2_loop} of Algorithm~\ref{alg:pref-graph}, there exist $q \in W$ and $q' \in W'$ such that $(q, q') \in Z$.
Therefore, by construction in Algorithm~\ref{alg:pref-graph}, for any $\varphi' \in \mathsf{MP}(\vec{q}')$, there is a $\varphi \in \mathsf{MP}(\vec{q})$ such that $\varphi'\weakpref \varphi$. Then, by Def.~\ref{def:preference-ltl}, $w' \succeq w$.

($\Leftarrow$) If $w'\succeq w$, then $(\vec{q},\vec{q}')\in Z$. Because $W \ne W'$, then there is no path from $\vec{q}' $ to $\vec{q}$. As a result, it is not the case that $w   \succeq w'$. Thus, $w\not\sim w'$.  

Case 3: proof similar to the proof of Case 2.

Case 4 is a direct consequence from Cases 1, 2 and 3.
	\end{proof}

Using the computational model \ac{pdfa}, we can directly compute the set $\{w\}^\uparrow$ for any $w\in \Sigma^\ast$.

\begin{lemma}
	\label{lem:wUp}
	For each word $w\in \Sigma^\ast$, if $\delta(\init, w)\in \prefnode$ for some $\prefnode\in \prefnodes$, then 
     the upper closure of $w$ is 
     \begin{multline} 
	 	\{ w\}^\uparrow = \{w' \in \Sigma^\ast \mid \exists \prefnode' \in \prefnodes,\\ \delta(\init, w')  \in \prefnode' \text{ and } \prefnode \rightsquigarrow \prefnode'\},
	\end{multline}
%
 and the lower closure of $w$ is 
 \begin{multline} 
	 	\{ w\}^\downarrow = \{w' \in \Sigma^\ast \mid \exists \prefnode' \in \prefnodes,\\ \delta(\init, w')  \in \prefnodes' \text{ and } \prefnode' \rightsquigarrow \prefnode\}
	\end{multline}
%
\end{lemma}
The lemma directly follows from the transition function in $\pdfa$ and the transitivity property of the preference relation and thus 
the proof is omitted. 
	 
	\begin{example}
        Consider three \ac{ltlf} formulas $\varphi_1 = \Eventually a$, $\varphi_2 = \Eventually b$, and $\varphi_3 = \neg \Eventually a \wedge \neg \Eventually b$. 
        Also, assume 
        $\varphi_1 \strictpref \varphi_2$, $\varphi_1 \strictpref \varphi_3$, and $\varphi_2 \nparallel \varphi_3$.
        The left column of Figure~\ref{fig:pdfa_constrcution} shows for each of the three LTL$_f$ formula, a DFA that encodes that formula.
        The column in right shows the \ac{pdfa} our algorithm constructs for these three formulas and the associated user preferences.
        In this \ac{pdfa}, we have written in blue for each state $x$, $\outcomes(x)$, the set of formulas satisfied when the word ends at state $x$.
        For each state $x$, we have also written in red, $\mathsf{MP}(x)$---the most preferred formulas among those formulas in $\outcomes(x)$.
        Accordingly,
        $\outcomes(qpr) = \{ \varphi_3 \}$, $\outcomes(q'pr')=\{\varphi_1 \}$, $\outcomes(qp'r')=\{\varphi_2\}$, and
        $\outcomes(q'p'r')=\{\varphi_1, \varphi_2\}$.
        Also, $\mathsf{MP}(qpr)=\{\varphi_3 \}$, $\mathsf{MP}(q'pr')=\{\varphi_1 \}$, $\mathsf{MP}(qp'r')=\{\varphi_2 \}$, and $\mathsf{MP}(q'p'r')=\{\varphi_1\}$.
        Note that because $\varphi_1 \strictpref  \varphi_2$, $\mathsf{MP}(q'p'r') = \mathsf{MP}(\{\varphi_1, \varphi_2\}) = \{\varphi_1 \}$.
        %
        Given that $\mathsf{MP}(q'p'r')=\mathsf{MP}(q'pr')$, states $q'p'r'$ and $q'pr'$ belong to the same vertex in the preference graph.
   %
	\end{example}
 
 \begin{figure}[h]
		\centering
		\includegraphics[width=1.0\linewidth]{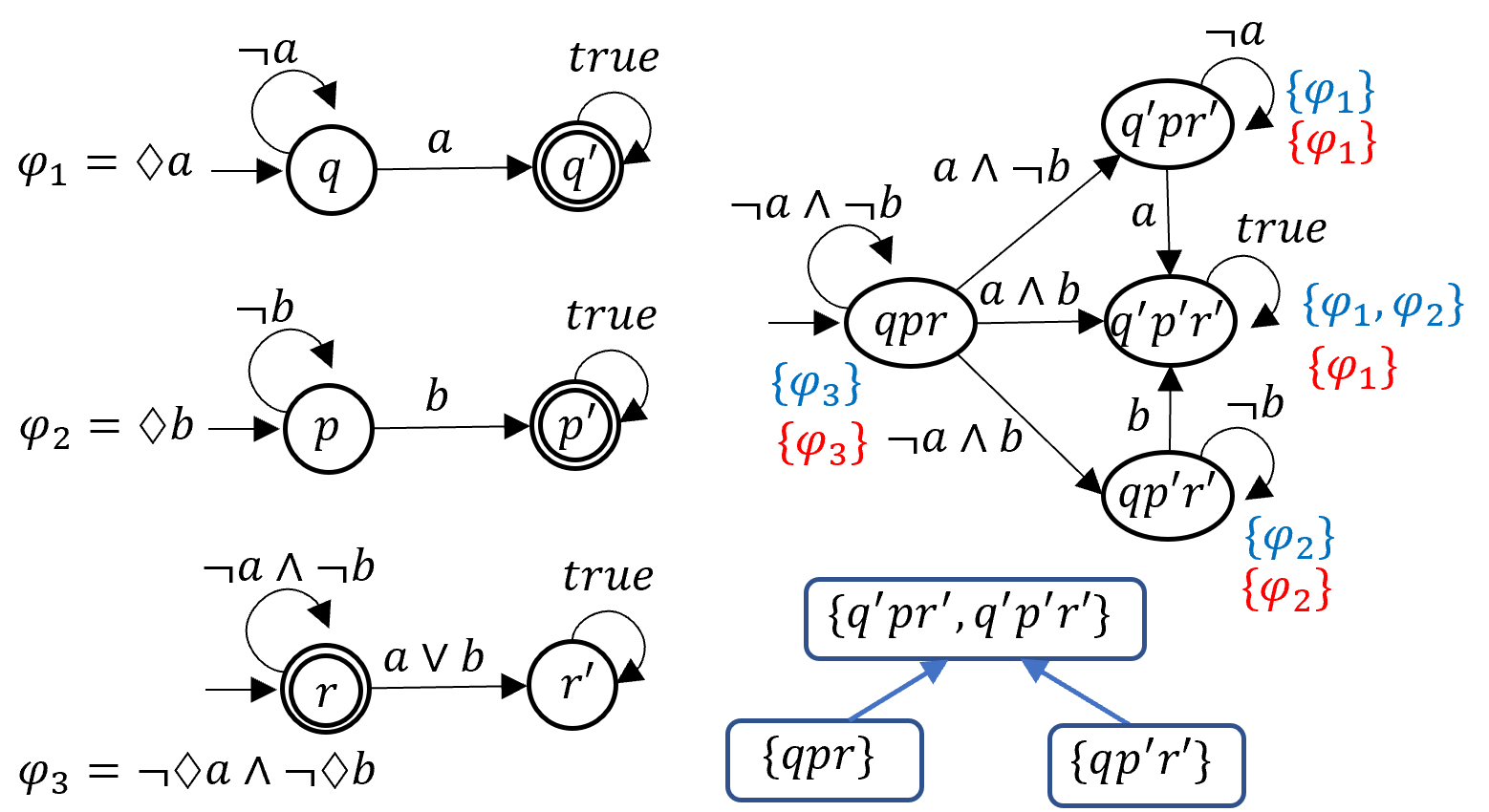}
		\caption{
		 PDFA for the example in Figure~\ref{fig:gap_garden}. 
		 \textbf{Left)} Three DFAs for three LTL$_f$ formulas $\varphi_1$, $\varphi_2$, and $\varphi_3$, for which the user preference is: $\varphi_1 \strictpref \varphi_2$,  
   $\varphi_1 \strictpref \varphi_3$, $\varphi_2 \nparallel \varphi_3$.
		 \textbf{Right)} The \ac{pdfa} constructed by our algorithm for the LTL$_f$ formulas and the preference over them. The output of each state---the set of formulas that satisfies every string that ends at that state---is shown in blue and the most preferred formulas for each state is shown in red.
		}
		\label{fig:pdfa_constrcution}
	\end{figure}

	\section{Synthesizing a  Most-Preferred Policy}
	\label{sec:alg}
  With the computational model \ac{pdfa} representing the partially-ordered temporal goals, we are ready to present a planning algorithm to solve Problem~\ref{problem:main-problem}.
The algorithm computes for a given \ac{tlmdp}, a policy that is most-preferred given the user preferences specified by a given \ac{pdfa}. 
The first step is to augment the planning state space with the states of the \ac{pdfa}. 
With this augmented state space, we can relate the preferences over traces in the \ac{mdp} to a preference over subsets of terminating states in a product \ac{mdp}, defined as follows.
	\begin{definition}[Product MDP]
	\label{def:prod}
	Let $M = \langle S, A:=\Sigma_{s \in S}A_s, \mathbf{P}, s_0, s_\bot, \calAP, L \rangle$ and 
	 $\pdfa = \langle Q, \Sigma, \delta, \init, G:=(\prefnodes, \prefedges)  \rangle$ be respectively the \ac{tlmdp} and the \ac{pdfa}. 
  The product of $M$ and $\pdfa$ is a tuple $\calM = (X, A:=\bigcup_{x \in X}A_x, \mathbf{T}, x_0, X_G, \calG := (\augnodes, \augedges))$ in which
	\begin{enumerate}
	    \item $X =  S \times Q$ is the state space;
	    \item $A$ is the action space, where for each $x=(s, q) \in X$, $A_x = A_s$ is the set of available actions at state $x$;
	    \item \label{itm:T} $\mathbf{T}: X \times A \rightarrow \dist(X)$ is the transition function such that for each state $(s, q) \in X$, action $a \in A$, and state $(s', q') \in X$;
	    \begin{multline*}
     \mathbf{T}((s, q), a, (s', q')) = \\  {
    \begin{cases}
        \mathbf{P}(s, a, s') & \mbox{\text{if} $q' = \delta(q, L(s')),$} \qquad\hfill \\
        0 & \text{otherwise};
    \end{cases}
    }
\end{multline*}
	    \item $x_0 = (s_0, \delta(\init, L(s_0)))$ is the initial state;
	    \item $X_G =  \{s_\bot \} \times Q$ is the set of terminating states;
	    \item \label{itm:pgraph} $\calG = (\augnodes, \augedges)$ is the preference graph, in which, letting $\augnode_i = \{s_\bot\}\times W_i$ for each $W_i \in \prefnodes$, 
	    \begin{itemize}
	        \item $\augnodes = \{\augnode_i\mid i=1,\ldots, \abs{\prefnodes}\}$ is the vertex set of the graph, and
	        \item $\augedges$ is the edge set of the graph such that $(\augnode_i, \augnode_j)\in \augedges$ if and only if $(W_i, W_j ) \in E$.
	    \end{itemize}
	    \end{enumerate}
	\end{definition}

 The preference graph of this MDP has been directly lifted from the one defined for the \ac{pdfa}. Given $Y, Y' \in \augnodes$, we use $Y \rightsquigarrow Y'$ to denote 
	 that $Y'$ is reachable from $Y$ in the preference graph $\calG$. Again, every $Y$ is reachable from itself.

 \begin{example}
 \label{ex:prod_pref_graph}
     Continuing with the example in Figure~\ref{fig:pdfa}, we have $\augnode_1 = \{ s_\bot\} \times W_1 = \{(s_\bot, q_2)\}$, $\augnode_2 =\{s_\bot\} \times  W_2=\{(s_\bot, q_4)\}$, $\augnode_3 = \{s_\bot \} \times W_3 = \{(s_\bot, q_1)\}$, and $\augnode_4 = \{s_\bot\} \times W_4 = \{(s_\bot, q_0), (s_\bot, q_3), (s_\bot, q_5)\}$.
 \end{example}

 Next, we show how to compute a stochastic nondominated policy in the sense of Definition~\ref{def:sto_dominance} through solving a \ac{momdp}. The existence of such a \ac{momdp} is guaranteed by the multi-utility representation theorem \cite[Proposition~1]{ok2002utility}, which states that for every partial order $\langle U, \succeq \rangle$ defined over a finite set $U$, there exists a vector-valued utility function $\mathbf{u}: U \rightarrow \mathbb{R}^n$ such that for any $x, y \in U$, $x \succeq y$ if and only if $\mathbf{u}(x) \geq \mathbf{u}(y)$ where $\geq$ is element-wise \footnote{Specifically, the multi-utility representation theorem \cite[Proposition~1]{ok2002utility} requires the partial order $\succeq$ over the set $U$ to be representable as an intersection of finitely many linear orders. However, \cite{dushnik1941partially,fishburn1985interval} have proved that every partial order over a finite set can be represented as an intersection of finitely many linear orders.}. 
 
      We extend the notions related to stochastic ordering for state subsets of the product MDP $\calM$, constructed  
		   in Definition~\ref{def:prod}, as follows:
     %

   \begin{definition}
     \label{def:prod_inc_dec_set}
     Let $\mathbf{\augnode} \subseteq \augnodes$ be a set of vertices in the preference graph $\mathcal{G}$. The upper closure of $\mathbf{\augnode}$ with respect to $\mathcal{G}$ is defined by 
     \[
	\mathbf{\augnode}_i^\uparrow = \{\augnode' \mid \exists \augnode \in \mathbf{\augnode}, \augnode\rightsquigarrow \augnode' \}
	\] and the lower closure of $\augnode$ is defined by
	\[
	\mathbf{\augnode}_i^\downarrow = \{\augnode' \mid \exists \augnode \in \mathbf{\augnode}, \augnode' \rightsquigarrow \augnode  \}.
	\] 
    Also, $ \mathbf{\augnode}  \subseteq \augnodes$ is called an \emph{increasing set} if $ \mathbf{\augnode}   = \mathbf{\augnode}  ^\uparrow$.
     \end{definition}
    %
    These sets are used to define a stochastic ordering type as follows:
	\begin{definition}
		Let $\mathfrak{E}_{st}(\augnodes), \mathfrak{E}_{wk}(\augnodes), \mathfrak{E}_{wk*}(\augnodes)$ denote the strong, weak, and weak* orderings, respectively, 
		where 
		\[
		\mathfrak{E}_{st}(\augnodes) = \{\text{ all increasing sets in $2^\augnodes$} \},
		\]
		\[
		\mathfrak{E}_{wk}(\augnodes) = \{\{\augnode\}^\uparrow \mid \augnode\in \augnodes \} \cup \{\augnodes, \emptyset\},
		\]
		\[
		\mathfrak{E}_{wk*}(\augnodes) = \{\augnodes\setminus \{\augnode\}^\downarrow \mid \augnode\in \augnodes \} \cup \{\augnodes, \emptyset\},
		\]
		\end{definition}   

%

For a stochastic ordering $\mathfrak{E} \in \{\mathfrak{E}_{st}, \mathfrak{E}_{wk}, \mathfrak{E}_{wk*}\}$, 
let the elements of set $\mathfrak{E}(\augnodes) \setminus \{\augnodes, \emptyset \}$ to be indexed arbitrary as $\mathbf{\augnode}_1, \mathbf{\augnode}_2, \cdots, \mathbf{\augnode}_N$.
We use this indexed set in the following construction to make a multi-objective MDP. 
    \begin{definition}[\ac{momdp}]
	\label{def:goalMDP}
	Given a stochastic ordering $\mathfrak{E} \in \{\mathfrak{E}_{st}, \mathfrak{E}_{wk}, \mathfrak{E}_{wk*}\}$, the multi-objective MDP (MOMDP) associated with the product MDP $\calM = \langle X, A, \mathbf{T}, x_0, X_G, \calG := (\augnodes, \augedges) \rangle$ in Definition~\ref{def:prod} and the stochastic ordering $\mathfrak{E}$, is a tuple $\calP = \langle X, A:=\bigcup_{x \in X} A_x, \mathbf{T}, x_0, X_G, \calZ=\{Z_1, Z_2, \cdots, Z_N \} \rangle$ in which $X$, $A$, $\mathbf{T}$, $x_0$, and $X_G$ are  the same elements in $\calM$ and for each $\mathbf{Y}_i \in \mathfrak{E}(\augnodes) \setminus \{\augnodes, \emptyset \}$, $Z_i = \bigcup_{Y \in \mathbf{Y}_i} Y$.
	The $i$-th objective in the \ac{momdp} is to maximize the probability for reaching the set $Z_i$.
	\end{definition}
  	Note that each $Z_i$ is a subset of goal states $X_G$, and that
    the intersection of two distinct goal subsets $Z_i$ and $Z_j$ may not be empty.

    \begin{remark}
    We exclude $\augnodes$ and $\emptyset$ from the construction of the multi-objective MDP. This is because we consider only proper policies and under any proper policy, any state in $X_G$ is reached with probability one. As a result, the probability of reaching the objectives $\emptyset$ and $\augnodes$ are always $0$ an $1$, respectively, regardless of the chosen stochastic ordering.
    \end{remark}
	 
    To illustrate the construction of $\mathcal{P}$, we continue with our running example.
\begin{example}
    Using the running  example in Figure~\ref{fig:pdfa}, for which the preference graph of the product MDP is shown in Example~\ref{ex:prod_pref_graph}, we have
    $\{ Y_1\}^\uparrow = \{ Y_1\}$, $\{Y_2 \}^\uparrow = \{Y_1, Y_2 \}$, $\{Y_3\}^\uparrow = \{Y_1, Y_3 \}$, and $\{Y_4\}^\uparrow = \{Y_1, Y_2, Y_3, Y_4 \}$, and as a result,
    \[
	\mathfrak{E}_{wk}(\augnodes) \setminus \{\augnodes, \emptyset \} = \{ \{Y_1\}, \{Y_1, Y_2\}, \{Y_1, Y_3\} \},
	\]
  and thus, under weak-stochastic ordering, the MOMDP will have the following objectives
    \[
	Z_1 = Y_1, Z_2 = Y_1 \cup Y_2, \text{and } Z_3 = Y_1 \cup Y_3.
	\]

Also, we have $\{Y_1\}^\downarrow = \{Y_1, Y_2, Y_3, Y_4 \}$, $\{Y_2\}^\downarrow = \{Y_2, Y_4 \}$, $\{Y_3\}^\downarrow = \{Y_3, Y_4 \}$, and
 \[
	\mathfrak{E}_{wk*}(\augnodes) \setminus \{\augnodes, \emptyset \}= \{ \{Y_1, Y_3\}, \{Y_1, Y_2\}, \{Y_1, Y_2, Y_3\} \},
	\]
 and therefore, under weak*-stochastic ordering, the MOMDP will have the following objectives 
    \[
	Z_1 = Y_1 \cup Y_3, Z_2 = Y_1 \cup Y_2, \text{and } Z_3 = Y_1 \cup Y_2 \cup Y_3.
	\]

 Furthermore, 
\[
	\mathfrak{E}_{st}(\augnodes) \setminus \{\augnodes, \emptyset \} = \{ \{Y_1\}, \{Y_1, Y_2\}, \{Y_1, Y_3\}, \{Y_1, Y_2, Y_3\} \},
	\]
 and hence, under strong-stochastic ordering, the MOMDP will have the following objectives
 \begin{multline}
     Z_1 = Y_1, Z_2 = Y_1 \cup Y_2, Z_3 = Y_1 \cup Y_3, \text{and } \\ Z_3 = Y_1 \cup Y_2 \cup Y_3. \notag
 \end{multline}

Accordingly, the objectives for each of the stochastic orderings in terms of the temporal goals in Figure~\ref{fig:gap_garden} can
be summarized as Table~\ref{tab:objectives}.
 
\end{example}

        
        

\begin{table}[h!]
	    \setlength{\tabcolsep}{0.2em}
    \centering    \captionof{table}{Objectives of different stochastic ordering types for the temporal goals in Figure~\ref{fig:gap_garden}.}

\begin{tabular}{cc}
        \hline \hline
          Stochastic Ordering &   Objectives \\ 
        \hline
        Weak &   $\{p_1\}, \{p_1, p_2\}, \{p_1, p_3\}$	 \\
        
         Strong &  $\{p_1\}, \{p_1, p_2\}, \{p_1, p_3\}, \{p_1, p_2, p_3\}$	 \\
        
         Weak*  &  $\{p_1, p_2\}, \{p_1, p_3\}, \{p_1, p_2, p_3\}$  \\
         \hline
        \end{tabular}
    \label{tab:objectives}
\end{table}

    Given the \ac{momdp} in Definition~\ref{def:goalMDP}, for a given randomized, finite-memory policy $\mu: X^* \rightarrow \dist(A)$, we can compute the value vector of $\mu$ as a $N$-dimensional vector $\mathbf{V}^\mu=[\mathbf{V}_1^\mu, \mathbf{V}_2^\mu, \cdots, \mathbf{V}_N^\mu]$ where for each $i$, $\mathbf{V}_i^\mu$ is the 
    probability of reaching states of $Z_i$ by following policy $\mu$, starting from the initial state.
  
    Given a randomized, memoryless policy $\mu: X \rightarrow \dist(A)$, to compute its value vector $\mathbf{V}^\mu$, we first set for each goal state $x_g \in X_G$, $\mathbf{V}^\mu(x_g)$ to be the vector such that for each $i \in \{1, \cdots, n \}$, $\mathbf{V}^\mu_i(x_g) = 1$ if $x_g \in Z_i$, and otherwise $\mathbf{V}^\mu_i(x_g) = 0$. Then we compute the values of the non-goals states $x \in X \setminus X_G$ via the Bellman equation
    \begin{equation}
        \mathbf{V}^{\mu}(x) = \sum_{a \in A} \left ( \mu(x, a) \sum_{x' \in X} \mathbf{T}(x, a, x') \mathbf{V}^{\mu}(x') \right ).
    \end{equation}

    \begin{definition}
    \label{def:paretoDomPolicy}
               Given two proper polices $\mu$ and $\mu'$ for $\calM$, it is said that $\mu$ \emph{Pareto dominates} $\mu'$, 
               denoted $\mu > \mu'$, if for each $i \in \{1, \cdots, N\}$, $\mathbf{V}_i^{\mu} \geq \mathbf{V}_i^{\mu'}$, and for at least one $j \in \{1, \cdots, n\}$, $\mathbf{V}_j^{\mu} > \mathbf{V}_j^{\mu'}$.
    \end{definition}
    Intuitively, $\mu$ Pareto dominates $\mu'$ if, compared to $\mu'$, it increases the probability of reaching at least a set $Z_j$ without reducing the probability of reaching other sets $Z_i$'s.

    \begin{definition}
  \label{def:nonDomPolicy}
             A proper policy $\mu$ for the MOMDP   is \emph{Pareto optimal} if for no proper policy $\mu'$ for the MOMDP it holds that $\mu' > \mu$.
    \end{definition}
    In words, a policy is Pareto optimal if it is not dominated by any policy.
   The \emph{Pareto front} is the set of all Pareto optimal policies. 
    %
     It is well-known that the set of memoryless policies suffices for achieving the Pareto front ~\citep{chatterjee2006markov}. Thus, we restrict to computing memoryless policies.

With this in mind, we present the following result.
    \begin{theorem}
    \label{thm:pareto-weakstochastic}
        Assume the MOMDP $\calP$ in Definition~\ref{def:goalMDP} is constrcuted under a stochastic ordering $\mathfrak{E} \in \{\mathfrak{E}_{st}, \mathfrak{E}_{wk}, \mathfrak{E}_{wk*}\}$ and let $\mu: X \rightarrow \dist(A)$ be a policy for $\calP$.
        Construct policy $\pi: S^* \rightarrow \dist(A)$ for the \ac{tlmdp} $M$ such that for each $\rho=s_0 s_1 \cdots s_n \in S^*$ it is set
        $\pi(\rho) = \mu((s_n, \delta(\iota, \trace(\rho))))$.
        If $\mu$ is Pareto optimal, then $\pi$ is $\mathfrak{E}$-stochastic nondominated, respecting the preference specified by \ac{pdfa} $\pdfa$.
    \end{theorem}
    \begin{proof}
    %
    We first provide a detailed proof of the case where $\mathfrak{E} = \mathfrak{E}_{wk}$, that is, where $\mathcal{P}$ is constructed for weak-stochastic ordering.
    We show that if $\mu$ is Pareto optimal then $\pi$ is weak-stochastic nondominated. To facilitate the proof, the following notation is used: Let $\Pr^\mu(\mbox{reach}(H), \calM)$ be the probability of terminating in the set $H \subseteq X$ given the policy $\mu$ for the MOMDP and $\Pr^\pi(\mbox{reach}(P), M)$ be the probability of terminating in the set $P \subseteq S$ given the policy $\pi$ in the original S\ac{tlmdp}. 
    
    

    First, consider that by the construction of the product MDP, Definition~\ref{def:prod}, preference graphs $\calG$ and $G$ are isomorphic, and thus, each $\augnode_i \in \augnodes$ is mapped to a single $\prefnode_i \in \prefnodes$, and vice versa.
    %
    Define $\prefnode_i^+ = 
    \bigcup_{\prefnode: \prefnode_i \rightsquigarrow \prefnode} \prefnode$ for each $\prefnode_i \in \prefnodes$. That is, let $\prefnode_i^+$ include the unions of states in all the nodes that can be reached from $\prefnode_i$ in the preference graph. Note that $\prefnode_i \in \prefnode_i^+$.
    Given that $\calG$ and $G$ are isomorphic, $\augnode_i \rightsquigarrow \augnode_j$ if and only if $\prefnode_i \rightsquigarrow \prefnode_j$ for all $i, j \in \{1, 2, \cdots, N \}$.
    This combined with that $Z_i = \bigcup_{\augnode \in \{\augnode_i\}^\uparrow} \augnode$ for $i \in \{1, \cdots, N\}$ by Definition~\ref{def:goalMDP}, implies that for each $i$,
    \begin{equation}
    \label{eq:Zi_Fi}
       \mathbf{V}_i^\mu =  \Pr^\mu(\reach{Z_i}, \calM) = \Pr^\pi(\mbox{reach}(\prefnode_i^+) , M\}.
    \end{equation}
    
    %
    Next, for each $w, w' \in \Sigma^\ast$ such that $\delta(\init, w)=\delta(\init, w')$, it holds that $\{w\}^\uparrow = \{w'\}^\uparrow$.
    Given this and Lemma~\ref{lem:wUp}, for each $\prefnode_i$ and $w \in \Sigma^\ast$ such that $\delta(\init, w) \in \prefnode_i$, 
    \begin{equation}
    \label{eq:Fi_W}
        \Pr^\pi(\mbox{reach}(\prefnode_i^+), M) =\Pr^\pi(\{w\}^\uparrow ).
    \end{equation}
    
    Finally, given that $\mu$ is a Pareto optimal policy, by Definition~\ref{def:paretoDomPolicy} and Definition~\ref{def:nonDomPolicy}, it means there exists no policy $\mu'$ such that 
     $\mathbf{V}_i^{\mu'} \geq \mathbf{V}_i^\mu$ for all integers $1 \leq i \leq n$ and 
    $\mathbf{V}_j^{\mu'} > \mathbf{V}_j^\mu$ for some integer $1 \leq j \leq n$.
    This, by (\ref{eq:Zi_Fi}) and (\ref{eq:Fi_W}) and that the set of randomized, memoryless policies suffices for the Pareto front of $\calM$, means there exists no policy $\pi' \in \Pi^M_{prop}$ such that $\Pr^{\pi'}(\{w\}^\uparrow ) \geq \Pr^{\pi}(\{w\}^\uparrow )$ for every $w \in \Sigma^*$ and  $\Pr^{\pi'}(\{w'\}^\uparrow ) > \Pr^{\pi}(\{w'\}^\uparrow )$ for some $w' \in  \Sigma^*$.
    This, by Definition~\ref{def:weak_dominating_policies} and Definition~\ref{def:sto_dominance}, means that $\pi$ is weak-stochastic nondominated.

Proof for the case where $\mathcal{P}$ is constructed for weak*-stochastic ordering, that is, where $\mathfrak{E} = \mathfrak{E}_{wk*}$, is very similar, except that wherever $\prefnode^+$ is used we use $\overline{\prefnode}^-$, defined as 
    $\overline{\prefnode}_i^- =\prefnodes\setminus \left(\bigcup_{\prefnode, \prefnode \rightsquigarrow \prefnode_i} \prefnode\right)$. 
For strong-stochastic ordering, we first construct for all subset $V\subseteq \prefnode$, the set $V^\uparrow = \{\prefnode \mid \exists \prefnode_i \in X, \prefnode_i \rightsquigarrow \prefnode \}$,  then exclude any set $V^\uparrow$ if $ V^\uparrow   \ne V$. Then, for each remaining set $V^\uparrow$, the set of states  contained in $V^\uparrow$ is used to define one reachability objective.
  
    \end{proof}

    Given the \ac{momdp}, one can use any existing methods to compute a set of Pareto optimal policies for $\calP$. 
    For a survey of those methods, see~\cite{roijers2013survey}.
    Note that computing the set of all Pareto optimal policies is generally infeasible, and thus, one needs to compute only a subset of them or to approximate them.

 \section{Case Study: Garden}
 \label{sec:caseStudy}
    In this section, we present the results from the planning algorithm for the running example in Figure~\ref{fig:gap_garden} .
In the garden, the actions of the robot are $N$, $S$, $E$, $W$--- corresponding to moving to the cell in the North, South, East, and West side of the current cell, respectively---and $T$ for staying in the current cell. The bee robot initially has a full charge, and using that charge it can fly only $12$ time steps.

\noindent \textbf{Uncertain environment: } A bird roams about  the south east part of the garden, colored yellow in the figure.
    When the bird and the bee are within the same cell, the bee needs to stop flying and hide in its current location until the bird goes away. The motion of the bird is given by a Markov chain. 
    Besides the stochastic movement of the bird, 
the weather is also stochastic and affects the robot's planning.     
    The robot cannot pollinate a flower while raining.
    We assume when the robot starts its task, at the leftmost cell at the bottom row, it is not raining and the probability that it will rain in the next step is $0.2$. 
    This probability increases for the consecutive steps each time by $0.2$ until the rain starts.
    Once the rain started, the probability for the rain to stop in the five following time steps will respectively be $0.2$, $0.4$, $0.6$, $0.8$, and $1.0$, assuming the rain has not already stopped at any of those time steps.

	We implemented this case study in Python and considered two variants of it: Case 1 without stochasticity in the robot's dynamics, and Case 2 with stochasticity. 
	In Case 1, when the robot decides to perform an action to move to a neighboring cell, its actuators will guarantee  that the robot will move to that cell after performing the action.
	In Case 2, the probability that the robot reaches the intended cell is $0.7$, and for each of the unintended directions except the opposite direction, the probably that the robot's actuators move the robot to that unintended direction is $0.1$. If the robot hits the boundary, it stays in its current cell.
	
	All the experiments were performed on a Windows 11 installed on a device with a core i$9$, 2.50GHz CPU and a 32GB memory. 
    \subsection{The Preference \ac{dfa}}
    We first describe how the \ac{pdfa} in Figure~\ref{fig:pdfa} is generated given the preference over \ac{ltlf} formulas $p_1$-$p_4$ from Example~\ref{ex:ltlf-goals}, formulating the preferences in Figure~\ref{fig:gap_garden}.
	Figure~\ref{fig:pdfa_experiment} shows for each of the four LTL$_f$ formula, a DFA that encodes that formula.
	This figure also shows the \ac{pdfa} our algorithm constructs for these four formulas and the associated user preferences, consisting of the underlying DFA and the preference graph, shown respectively in Part (e) and Part (f) of this figure.
        %
    %
	The \ac{pdfa} is generated using our open-source tool \footnote{Tool for constructing \ac{pdfa} from a preference over \ac{ltlf} formulas: \url{https://akulkarni.me/prefltlf2pdfa.html}} implemented in Python3. 
        Note that the \ac{pdfa} in this figure is the same \ac{pdfa} in Figure~\ref{fig:pdfa}, but this figure illustrates how the states $q_0$-$q_5$ in Figure~\ref{fig:pdfa} are constructed.
	The states in the underlying \ac{dfa} represent the formulas satisfied by any word whose trace ends in that state.
	For instance, $\outcomes((3, 2, 3, 4)) = \{p_4\}$. This is because the first three components $3, 2, 3$ are not accepting in their respective \ac{dfa}s, but the fourth component, $4$, is an accepting state in the \ac{dfa} for $p_4$ (see Figure~\ref{fig:pdfa_experiment}). 
	The states of the preference graph encode the partition of $Q$ based on the most-preferred outcomes satisfied by the states. 
	For example, state $(3, 6, 3, 2)$ belongs to the block $(0, 1, 0, 0)$ since the most-preferred outcomes satisfied by $(3, 6, 3, 2)$ is $\{p_2\}$.

    Recall how the \ac{pdfa} produces a ranking over the words in $\Sigma^*$.
	For example, consider two paths in the \ac{mdp} $M$: A path $\rho_1$ that pollinates tulips first and then daisies, and a path $\rho_2$ that first pollinates orchids and then daisies. 
	The trace of $\rho_1$, $\trace(\rho_1)$, induces a path that terminates in state $(4, 5, 3, 2)$ of the underlying DFA, whereas $\trace(\rho_2)$ terminates in $(3, 6, 3, 2)$. 
	Accordingly, state $(4, 5, 3, 2)$ of the \ac{pdfa} belongs to the block of the partition represented by the vertex $(1, 0, 0, 0)$ in the preference graph, and state $(3, 6, 3, 2)$ belongs to the block represented by the vertex $(0, 1, 0, 0)$.
	Since the preference graph has an edge from $(0, 1, 0, 0)$ to $(1, 0, 0, 0)$, it is implied that $\rho_1$ is strictly preferred over $\rho_2$.

 \begin{figure*}[t]
		\centering
		\begin{subfigure}[b]{0.24\textwidth}
			\centering
			\includegraphics[scale=0.3]{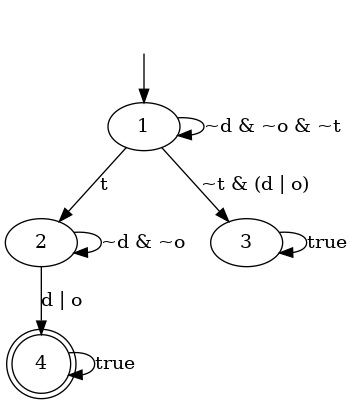}
   \caption{DFA for $p_1$: $(\neg d \wedge \neg o) \until (t \wedge \Next \Eventually ( d \vee o))$.}
		\end{subfigure}
		\hfill
		\begin{subfigure}[b]{0.44\textwidth}
			\centering
			\includegraphics[scale=0.3]{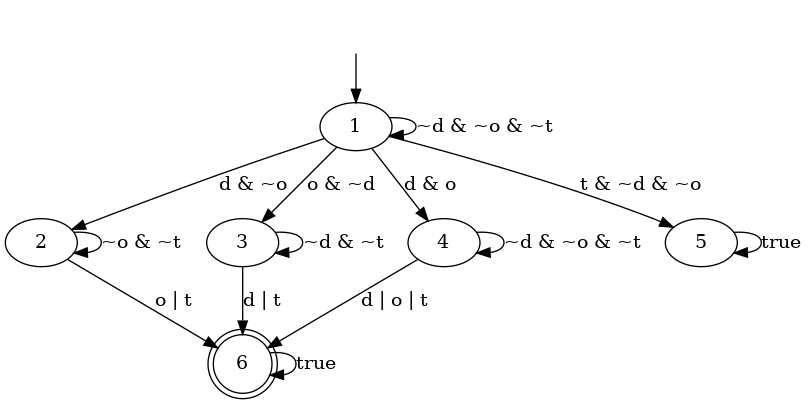}
   \caption{DFA for $p_2$: $\neg t \until ((o \wedge \Next \Eventually (d \vee t)) \vee (d \wedge \Next \Eventually ( o \vee  t)))$.}
		\end{subfigure}
  \hfill
		\begin{subfigure}[b]{0.3\textwidth}
			\centering
			\includegraphics[scale=0.3]{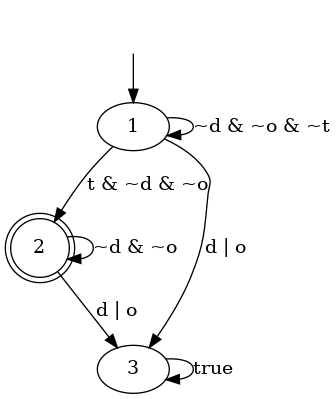}
   \caption{DFA for $p_3$: $(\neg d \wedge \neg o) \until (t \wedge \Always (\neg d \wedge \neg o))$.}
		\end{subfigure}
    \hfill
		\begin{subfigure}[b]{0.27\textwidth}
			\centering
			\includegraphics[scale=0.285]{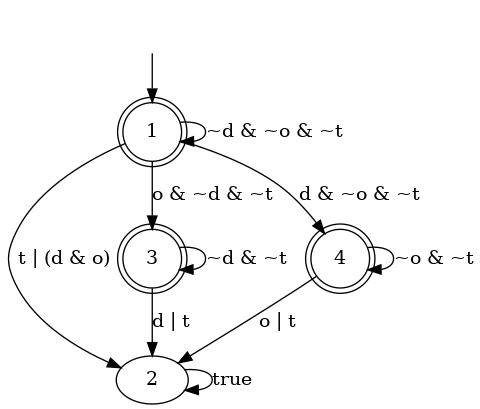}
   \caption{DFA for $p_4$: $\Always(\neg d \land \neg o \land \neg t) \lor (\Eventually o \land \Always(\neg d \land \neg t)) \lor  (\Eventually d \land \Always(\neg o \land \neg t))$.}
		\end{subfigure}
  \hfill
		\begin{subfigure}[b]{0.44\textwidth}
			\centering
			\includegraphics[scale=0.29]{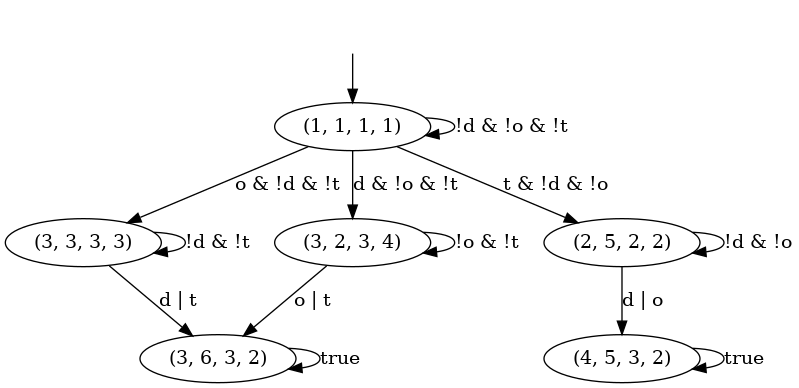}
			\caption{The underlying DFA of the \ac{pdfa}.}
		\end{subfigure}
  \hfill
		\begin{subfigure}[b]{0.28\textwidth}
			\centering
			\includegraphics[scale=0.285]{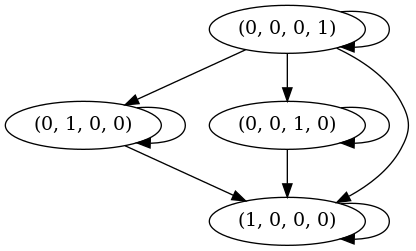}
			\caption{The preference graph of the \ac{pdfa}.}

		\end{subfigure}  
		\caption{\textbf{a)-d)} The DFAs for $p_1-p_4$ for the example Figure~\ref{fig:gap_garden}, which are constrcuted
             by our online tools, available at \url{https://akulkarni.me/prefltlf2pdfa.html}.
  \textbf{e)-f)} The PDFA for the example in Figure~\ref{fig:gap_garden}, which is constructed the implementation of our algorithm for converting a preference model over \ac{ltlf} formulas into a \ac{pdfa}.}
  \label{fig:pdfa_experiment}
	\end{figure*}

	\subsection{Case 1: Deterministic  Action Transitions but   Uncertain Environments}
 For the case when the robot's actions have deterministic outcomes,	the constructed MDP   has $10,460$ states and $280,643$ transitions (its transition function has $280,643$ entries with non-zero probabilities).
	It took $39.06$ seconds for our program to construct the MDP.
	The product MDP had $36,649$ states and $946,467$ transitions.
	The average construction times for the product MDP over $10$ constructions for each of the weak stochastic order, strong stochastic order, and weak* stochastic order were respectively $238.78$, $238.57$, and $238.76$ seconds.

	 Given the preference described in Fig.~\ref{fig:pdfa}, we employ linear scalarization methods to solve the \ac{momdp}.  
Specifically, given a weight vector  $\mathbf{w} \in [0, 1]^N$ , we compute the nondominated policy $\mu_{\mathbf{w}}$, by first setting $V_{\mathbf{w}}(x) = \sum_{1 \leq i \leq N: x \in Z_i} \mathbf{w}[i] $ for each goal state $x \in X_G$, and then by solving the following Bellman equation for the values of  the non-goal states $x \in X \setminus X_G$:
    %
    \begin{equation}
        V_{\mathbf{w}}(x) = \max_{a \in A_x} \sum_{x' \in X} \mathbf{T}(x, a, x')V_{\mathbf{w}}(x').
    \end{equation}

    The policy for those states is recovered from $V_{\mathbf{w}}(x)$ as
    \begin{equation}
        \mu_{\mathbf{w}}(x) = \argmax_{a \in A_x} \sum_{x' \in X} \mathbf{T}(x, a, x')V_{\mathbf{w}}(x').
    \end{equation}

    For each of the three stochastic orderings $\mathfrak{E} \in \{\mathfrak{E}_{st}, \mathfrak{E}_{wk}, \mathfrak{E}_{wk*}\}$, we randomly generated $1,000$ weight vectors and used each one of them to compute a Pareto optimal policy for the MOMDP.  
    For each stochastic orderings $\mathfrak{E}$, the $1,000$ computed Pareto-optimal policies are expected to be $\mathfrak{E}$-nondominated policies.
   From the result, it is noted that for each stochastic orderings $\mathfrak{E}$, none of those $1,000$ computed polices were $\mathfrak{E}$-stochastic dominated by the other polices. This is expected due to Theorem~\ref{thm:pareto-weakstochastic}.

    Next, we provide more detailed analysis for   weak-stochastic non-dominated policies. 
    Recall from Table~\ref{tab:objectives} that the objectives for the weak-stochastic ordering are $\{p_1\}$, $\{p_1, p_2\}$, and
    $\{p_1, p_3\}$.
    Since it is difficult to illustrate the 3D Pareto front, we select a set of policies with similar probabilities (approximately 0.24) of satisfying $p_1$ and then plot the values of these policies for the objectives $\{p_1, p_2\}$ and $\{p_1, p_3\}$. Figure~\ref{fig:case_study_pareto_front} shows the values of those  policies for the objectives $\{p_1, p_2\}$ and $\{p_1, p_3\}$. 
   This figure shows that none of those policies weak-stochastic dominates each other. 

    \begin{figure}[h]
		\centering
		\includegraphics[width=1.0\linewidth]{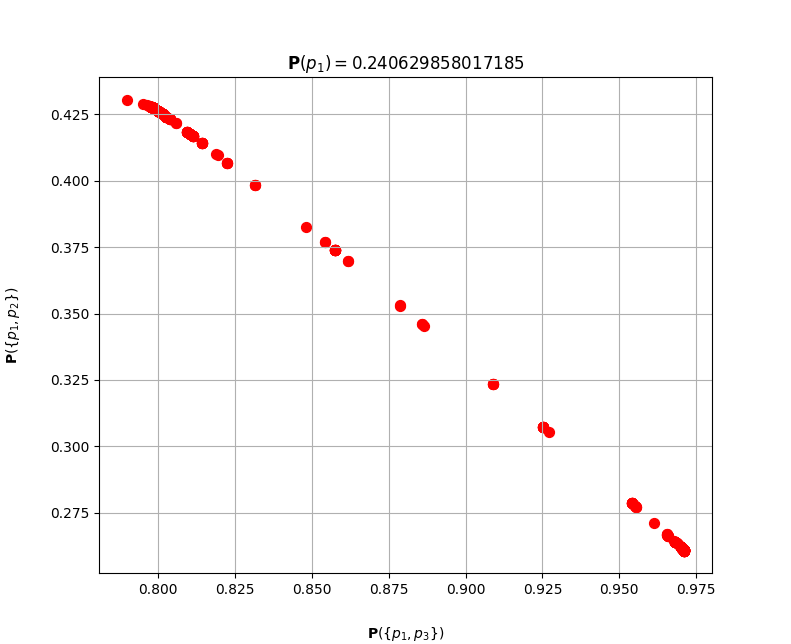}
		\caption{
            The probabilities of satisfying objectives $\{p_1, p_2\}$ and $\{p_1, p_3\}$ by the computed policies for weak-stochastic ordering who satisfy the objective $\{p_1\}$ with probability $0.2406$. 
		}
		\label{fig:case_study_pareto_front}
	\end{figure}

              	    \begin{table}[h!]
	    \setlength{\tabcolsep}{0.2em}
    \centering
    \begin{tabular}{cccc}
    \hline \hline
     & \scriptsize     Weight Vector & \scriptsize \makecell{Value Vector \\ $[\{p_1\}, \{p_1, p_2\}, \{p_1, p_3\}]$} & \scriptsize  Prob. of individual outcomes \\ 
    \hline
   \scriptsize 1 &  \scriptsize [0.466, 0.412, 0.122] & \scriptsize 	[0.110, 0.799, 0.291] & \scriptsize [0.110, 0.689, 0.181, 0.020] \\
   \scriptsize 2 &  \scriptsize [0.363, 0.438, 0.199] & \scriptsize 	[0.146, 0.726, 0.395] & \scriptsize [0.146, 0.580, 0.249, 0.025] \\
    \scriptsize 3 &  \scriptsize [0.207, 0.484, 0.309] & \scriptsize 	[0.201, 0.558, 0.633] & \scriptsize [0.201, 0.357, 0.432, 0.01] \\
    \scriptsize 4 &  \scriptsize [0.134, 0.519, 0.347] & \scriptsize 	[0.173, 0.638, 0.527] & \scriptsize [0.173, 0.465, 0.354, 0.008] \\
    \scriptsize 5 &  \scriptsize [0.141, 0.541, 0.318] & \scriptsize 	[0.068, 0.874, 0.187] & \scriptsize [0.068, 0.806, 0.119, 0.007] \\
    \scriptsize 6 &  \scriptsize [0.434, 0.339, 0.227] & \scriptsize 	[0.241, 0.427, 0.798] & \scriptsize [0.241, 0.186, 0.557, 0.016] \\
    \scriptsize 7 &  \scriptsize [0.428, 0.223, 0.349] & \scriptsize 	[0.239, 0.250, 0.980] & \scriptsize [0.239, 0.011, 0.741, 0.009] \\   
    \scriptsize 8 &  \scriptsize [0.213, 0.395, 0.392] & \scriptsize 	[0.240, 0.307, 0.925] & \scriptsize [0.240, 0.067, 0.685, 0.008] \\
     \scriptsize 9 &  \scriptsize [0.742, 0.208, 0.050] & \scriptsize 	[0.240, 0.432, 0.787] & \scriptsize [0.240, 0.192, 0.547, 0.021] \\
     \scriptsize 10 &  \scriptsize [0.057, 0.488, 0.455] & \scriptsize 	[0.240, 0.398, 0.831] & \scriptsize [0.240, 0.158, 0.591, 0.011] \\
     \hline
    \end{tabular}
        \caption{Ten weak-stochastic nondominated polices computed by our algorithm for the Garden case study. 
    }
    \label{tbl:nondom_policies}
    \end{table} 
    Table~\ref{tbl:nondom_policies} shows $10$ out of those $1,000$ weight vectors along with the value vectors of the weak-stochastic non-dominated polices computed for those weight vectors and the corresponding probabilities those polices assign to the four preferences $p_1$ through $p_4$.
    For each policy, the last column shows a probability distribution over individual outcomes $p_1,\ldots, p_4$ indicating the probabilities of satisfying those formulas (in that order), given the computed policy. 
    The third column shows the multi-objective value vector of each computed policy. It is noted that
    none of those value vectors dominates any other value vector.

Rows $6$ and $9$ of this table show that even if the weight assigned to the most preferred outcome, $p_1$, is significantly higher than the weights assigned to the other preferences, the probability that $p_1$ to be satisfied is still less than $0.25$. 
    This is justified by the fact that the robot's battery capacity supports the robot for only $12$ time steps and thus to achieve $p_1$, the robot must not be stopped by the bird and not encounter rain when it reaches a cell to do pollination. The probability to satisfy these conditions given the environment dynamics is less than $0.25$.

    The probability of $p_4$ to be satisfied in any entry of this table is very small, regardless of the weight vector. This is because $p_4$ has the lowest priority, and any policy would prefer to satisfy other preferences who are assigned higher priorities.

   Although the objectives $\{p_1, p_2\}$ and $\{p_1, p_3\}$ in the eighth and the tenth rows are treated almost equally by the weight vector in terms of importance, the probability that the later to be satisfied is significantly bigger than the probability of the former to be satisfied. This is because the objective $\{p_1 \}$ contains the preference with the highest priority and that those two rows assign a very high weight to this objective, forcing the policy to try to satisfy $p_1$. Further, by attempting to perform $p_1$, the robot has the chance to accomplish $p_3$ within the same attempt, even if it fails to accomplish $p_1$. 
    %
    More precisely, if in attempting to perform the task $p_1$---first tulips and then at least one out of daisies and orchids---the robot succeeds to pollinate the tulips but fails to pollinate the daisies and orchids, then it has already accomplished $p_3$, even though it has failed in accomplishing what it was aiming for---$p_1$.

      \subsection{Case 2: Introducing Stochastic    Robot Dynamics}
	The MDP for this variant has the same number of states, $10,460$, but it has more transitions, $779,396$, which is due to  the stochasticity in   robot's dynamics.
    We consider all the three types of stochastic orderings.
    The MDP construction time for weak-stochastic ordering, strong-stochastic ordering, and weak*-stochastic ordering were respectively $205.92$, $203,43$
    and $238.83$ seconds. The construction time of the product MDP for these three stochastic ordering types were respectively $1,088.58$, $1,088.59$, and $1,089.02$ seconds.
    
    %
	
    Due to the stochasticity in the robot's dynamic, for each kind of stochastic ordering, we expect the policy computed for a specific weight vector 
	to perform ``poorer'' than a policy computed for the same weight vector of Case 1.
	We first compare the two polices for weak-stochastic ordering, given the weight vector $[0.3333, 0.3333, 0.3334]$.
	The probabilities of the preferences to be satisfied for the variant without stochasticity were $[p_1: 0.241, p_2: 0.053, p_3:0.699, p_4: 0.007]$, while those probabilities for the variant with stochasticity were $[p_1: 0.008, p_2: 0.821, p_3:0.120, p_4: 0.051]$.
	While the former policy yields a higher probability of achieving $p_3$, the latter policy puts most of its efforts to satisfy $p_2$ and has a very low probability (close to 0) to satisfy the most preferred goal $p_1$.
Similar observation is made for strong-stochastic ordering using the weight vector $[0.25, 0.25, 0.25, 0.25]$.
    The probabilities of the preferences to be satisfied for the variant without stochasticity were $[p_1: 0.241, p_2: 0.053, p_3: 0.699, p_4: 0.007]$, while those probabilities for the variant with stochasticity were $[p_1: 0.002, p_2: 0.931, p_3: 0.026, p_4: 0.041]$. 

    Lastly, for weak*-stochastic ordering, we have three  objectives $\{p_1, p_2\}$, $\{p_1, p_3\}$, $\{p_1, p_2, p_3\}$ (after removing the empty set and 
     the set  $\{p_1, p_2, p_3, p_4\}$). Given the weight vector $[0.3333, 0.3333, 0.3334]$, 
    the probabilities of the preferences to be satisfied for the variant without stochasticity were $[p_1: 0.241, p_2: 0.053, p_3: 0.699, p_4: 0.007]$, while those probabilities for the variant with stochasticity were $[p_1: 0.000, p_2: 0.953, p_3: 0.007, p_4: 0.040]$. 

    Given the same weight vector but different stochastic orderings, we observed that the probability of satisfying $p_2$ in the stochastic variant is much larger ($\approx$ 16 times more likely)  than that of the deterministic variant. This result is mainly due to the difficulty in reaching tulips given the coupled inherent stochastic dynamics and uncertain environmental factors (clouds and the bird). Because the chance of reaching tulips is very small, the probability of satisfying $p_1$ or $p_3$ -- both require tulips to be visited --- are equally small. As a result, the preference-based planner (across all three stochastic orders)  satisfies $p_2$ with a much higher probability since $p_2$ only requires two flower types to be pollinated. This experimental comparison demonstrates the flexibility of preference-based planners to adjust the goal based on changes in the system and environment dynamics.
 

    \section{Conclusions and Future Work}
    In this paper, we introduced a formal language for specifying user's partially-ordered preferences over temporal goals expressed in LTL$_f$. We developed an algorithm to convert the user preference over LTL$_f$ formulas into an automaton with a preorder over the acceptance conditions. To synthesize a most preferred policy in a stochastic environment, we utilized stochastic ordering to translate a partially-ordered user's preference to a preorder over probabilistic distributions over the system trajectories. This allowed us to rank the policies based on the partially-ordered user's preference.     
    Leveraging the automaton structure, we proved that computing a most-preferred policy can be reduced to finding a Pareto-optimal policy in a multi-objective \ac{mdp} augmented with the automaton states.

This work provides fundamental algorithms and principled approach for preference-based probabilistic planning with partially-ordered temporal logic objectives in stochastic systems.   
    A direction for future work will be to extend the planning with preference over temporal goals that are satisfied in infinite time, for instance, recurrent properties and other more general properties in temporal logic. For robotic applications, it would be of practical interest to design a conversational-AI interface that elicits  human preferences and translating natural language specifications into the preference model, and thus facilitate human-on-the-loop planning. 
    \section*{Funding}

The author(s) disclosed receipt of the following financial support for the research, authorship, and/or publication of this article: This work was supported by   the Air Force Office of Scientific Research under award number FA9550-21-1-0085 and in part by NSF under award numbers 2024802 and 2207759.
 
\bibliographystyle{SageH}
	\bibliography{sample}

\end{document}